\documentclass[preprint,12pt]{elsarticle}

\usepackage[numbers]{natbib}

\usepackage{amssymb}

\usepackage{amsmath}

\usepackage{mathtools}
\usepackage{amsthm}
\usepackage{multicol}
\usepackage{amsthm}
\usepackage{multirow}
\DeclareMathOperator*{\mean}{\mathbb{E}}
\DeclareMathOperator*{\var}{\mathbb{V}}
\DeclareMathOperator*{\argmin}{argmin}

\newcommand\given[1][]{\:#1\vert\:}
\usepackage{booktabs}
\usepackage{subfigure}
\usepackage{pifont}

\usepackage{graphicx}

\usepackage{soul}

\newcommand{\ie}{\textit{i.e.}}

\theoremstyle{plain}
\newtheorem{theorem}{Theorem}[section]

\newtheorem{lemma}[theorem]{Lemma}

\theoremstyle{definition}

\theoremstyle{remark}

\begin{document}

\begin{frontmatter}

\title{Towards a Better Evaluation of Out-of-Domain Generalization}

\author[naveraddress]{Duhun Hwang\fnref{eqaulcontr}}
\ead{duhun.hwang@navercorp.com} 

\author[snuaddress]{Suhyun Kang\fnref{eqaulcontr}}
\ead{su_hyun@snu.ac.kr} 

\author[lgaiaddress]{Moonjung Eo}
\ead{moonj@lgresearch.ai} 

\author[snuaddress]{Jimyeong Kim}
\ead{wlaud1001@snu.ac.kr} 

\author[snuaddress,snuaddress2]{Wonjong Rhee\corref{mycorrespondingauthor}}
\cortext[mycorrespondingauthor]{Corresponding author}
\ead{wrhee@snu.ac.kr}

\fntext[eqaulcontr]{Authors contributed equally as first author.}

\affiliation[naveraddress]{organization={Shopping Foundation Models Team, NAVER},
            country={South Korea}}
\affiliation[snuaddress]{organization={Department of Intelligence and Information, Seoul National University},
            country={South Korea}
            }
\affiliation[lgaiaddress]{organization={Data Intelligence Lab., LG AI Research},
            country={South Korea}}
\affiliation[snuaddress2]{organization={IPAI and RICS, Seoul National University}, 
            country={South Korea}}

\begin{abstract} 
The objective of Domain Generalization (DG) is to devise algorithms and models capable of achieving high performance on previously unseen test distributions. In the pursuit of this objective, average measure has been employed as the prevalent measure for evaluating models and comparing algorithms in the existing DG studies. Despite its significance, a comprehensive exploration of the average measure has been lacking and its suitability in approximating the true domain generalization performance has been questionable. In this study, we carefully investigate the limitations inherent in the average measure and propose worst+gap measure as a robust alternative. We establish theoretical grounds of the proposed measure by deriving two theorems starting from two different assumptions. We conduct extensive experimental investigations to compare the proposed worst+gap measure with the conventional average measure. Given the indispensable need to access the true DG performance for studying measures, we modify five existing datasets to come up with SR-CMNIST, C-Cats\&Dogs, L-CIFAR10, PACS-corrupted, and VLCS-corrupted datasets. The experiment results unveil an inferior performance of the average measure in approximating the true DG performance and confirm the robustness of the theoretically supported worst+gap measure. 
\end{abstract}


\end{frontmatter}

\section{Introduction}
\label{sec:introduction}

While deep learning has achieved remarkable success across diverse domains, it remains susceptible to the challenge of Domain Generalization (DG). DG encapsulates the predicament of devising models that can effectively extend their performance to unseen test datasets. Therefore, DG involves the task of creating models capable of exhibiting robust generalization in previously unseen environments.
As formally stated in \cite{arjovsky2019invariant, ahuja2021invariance,robey2021model}, the goal of DG algorithms is to use multiple training datasets to find a model $f$, as follows:
\begin{equation} 
\label{eq:goal}
    \argmin_{f}\max\limits_{e \in \mathcal{E}_{\text{all}}}\mathcal{R}_{e}(f),
\end{equation}
where $\mathcal{R}_{e}(f)$ is the error rate in environment $e$ and $\mathcal{E}_{\text{all}}$ is the set of all possible environments including the training environments.

To address the DG problem, previous studies~\cite{ahuja2021invariance,pezeshki2021gradient,mahajan2021domain,zhou2021examining} have exploited the concepts of spurious correlation that should be discarded and invariant correlation that should be learned.
Furthermore, recent studies have proposed various algorithms to learn invariant correlations.
Some have tried to learn the environment-invariant representations~\cite{arjovsky2019invariant, sagawa2019distributionally,sun2016deep,li2018domain,ganin2016domain,li2018deep,blanchard2011generalizing,huang2020self, blanchard2021domain}.
Data augmentation techniques have also been exploited to avoid environment overfitting~\cite{yan2020improve, wang2020heterogeneous,nam2019reducing,krueger2021out}.
Other works have borrowed learning-to-learn strategies from meta-learning to deal with environment-shifts~\cite{li2018learning,zhang2020adaptive}.

Despite the extensive efforts, \citet{gulrajani2020search} demonstrated that a carefully implemented ERM~(Empirical Risk Minimization) can achieve a competitive performance when compared to the State-of-The-Art (SoTA) DG algorithms.
The competitive performance of ERM, however, is unexpected because ERM is known for its vulnerability to distributional shift~\cite{hashimoto2018fairness, zhai2022understanding, liu2021just}. 
To understand this paradoxical situation, we compare ERM with the true best algorithm in three different ways using ideal, worst+gap, and average as the measure. 
Ideal measure serves as the true DG performance (i.e., oracle performance) that is compatible with Eq.~(\ref{eq:goal}), and it will be formally defined in Section~\ref{sec:ideal_measure}.
Average measure has been predominantly used as the evaluation measure of DG~\cite{sun2016deep,krueger2021out,muandet2013domain,li2019episodic,seo2020learning,gulrajani2020search,ahuja2021invariance}, and it will be formally defined in Section~\ref{sec:average_measure}. 
The results are shown in Figure~\ref{figure:ERM}. When the ideal measure is used, ERM performs much worse than the best algorithm (Figure~\ref{figure:ERM}(a)). This is consistent with the known vulnerability of ERM. 
When the average measure is used, however, both algorithms are assessed to achieve a similar performance (Figure~\ref{figure:ERM}(c)). This is not because they truly perform at a similar level. It is simply because the average measure is not consistent with the ideal measure.

\begin{figure*}[t!]
% \centering
% \resizebox{1.\columnwidth}{!}{
% \begin{tabular}{ccc}
%     \includegraphics[width=0.31\columnwidth]{FIG/erm_IDEAL_4_error_rate.png} & \includegraphics[width=0.31\columnwidth]{FIG/erm_WG_4_error_rate.png} & \includegraphics[width=0.31\columnwidth]{FIG/erm_AVG_4_error_rate.png} \\
%     \multicolumn{1}{c}{(a) True error rate} &
%     \multicolumn{1}{c}{(b) Worst+gap measure (ours)} &
%     \multicolumn{1}{c}{(c) Average measure} \\
% \end{tabular}
% }
  \centering
  \subfigure[True error rate]{
    \includegraphics[width=0.24\textwidth]{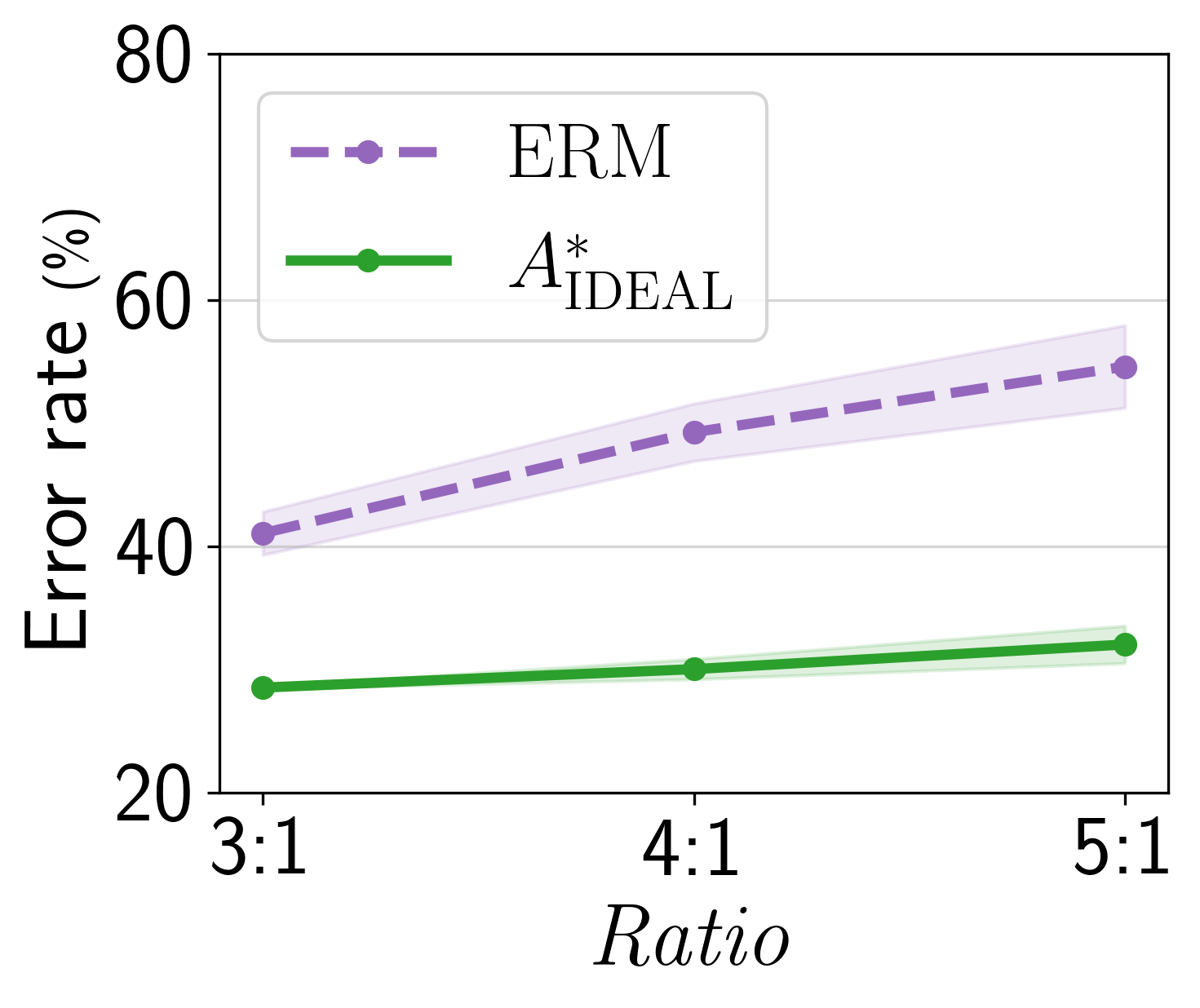}
  }
  \hspace{0.5cm}
  \subfigure[Worst+gap measure (ours)]{
    \includegraphics[width=0.24\textwidth]{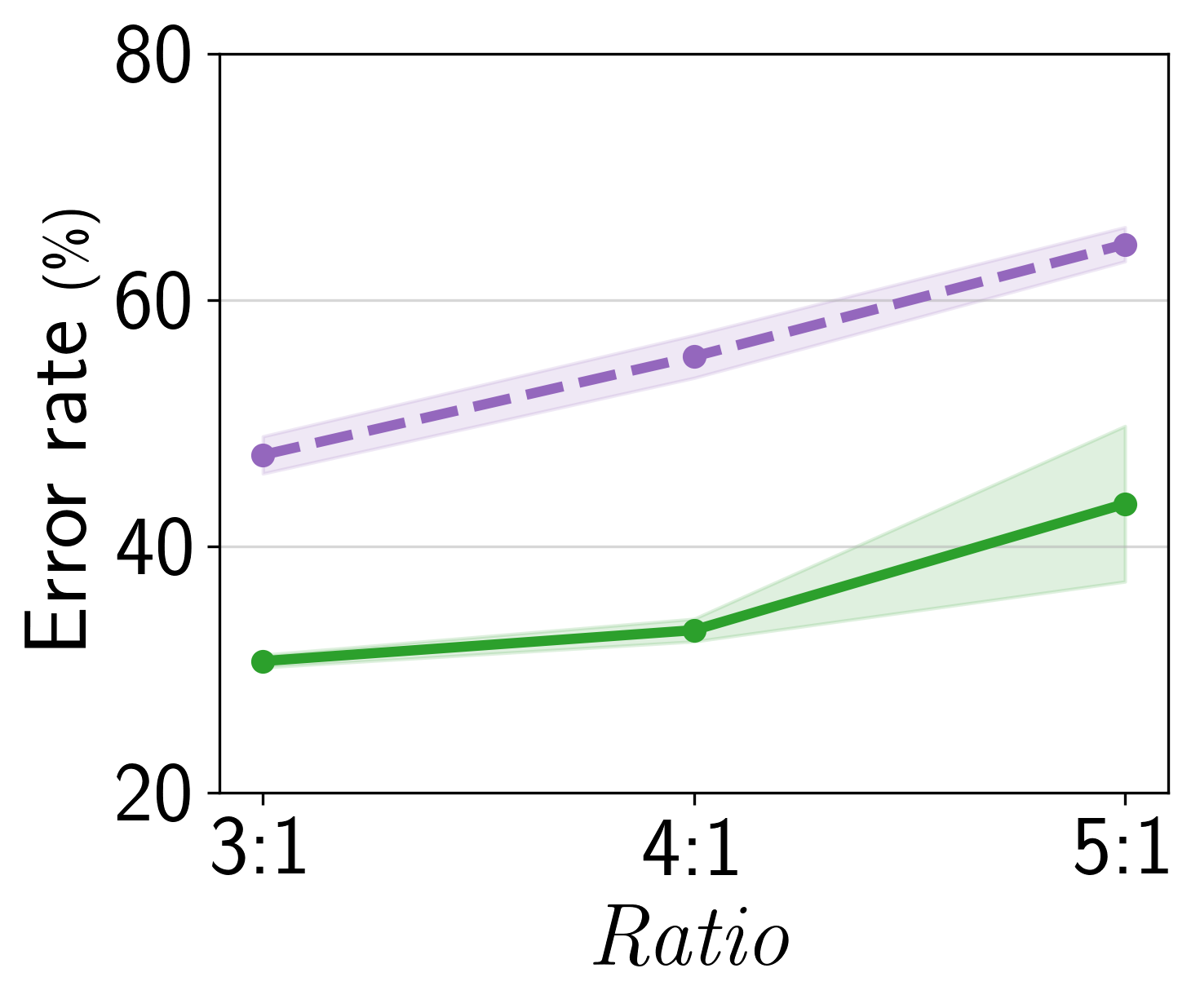}
  }
  \hspace{0.5cm}
  \subfigure[Average measure]{
    \includegraphics[width=0.24\textwidth]{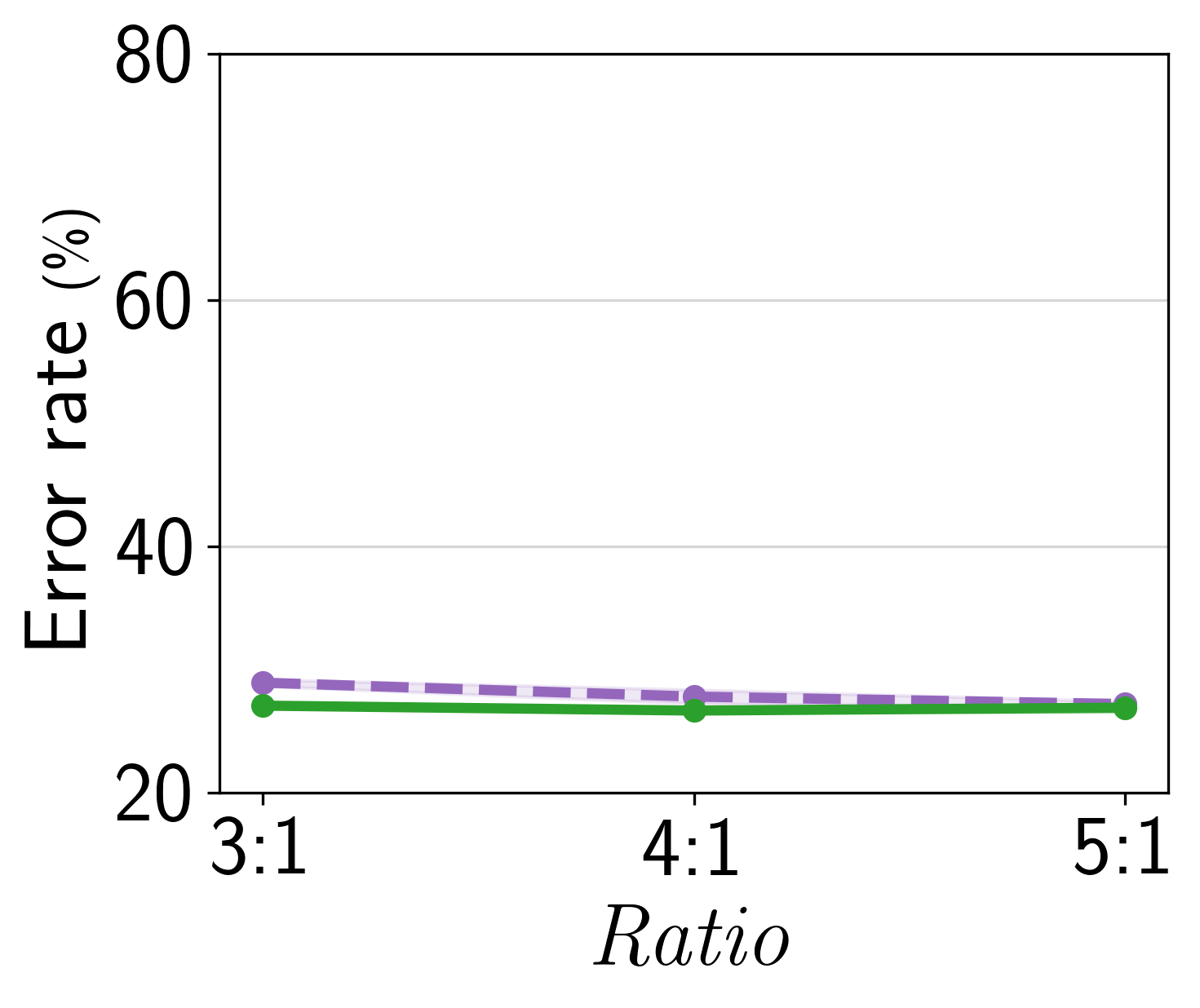}
  }
\caption{Comparison between ERM and the best performing algorithm~($A^{*}_{\text{IDEAL}}$) for different \textit{Ratio} configurations of SR-CMNIST dataset. For each plot, (a) true error rate, (b) worst+gap measure, or (c) average measure is used as the measure. 
Compared to the true error rate shown in (a), the average measure shown in (c) distorts the assessment. 
The results for \textit{Scale}$=4$ are shown. The results for the other \textit{Scale} values can be found in Section~\ref{sec:fulltable_erm}.
}
\label{figure:ERM}
\end{figure*}

Predictive power and invariance over domain shifts have been identified as the core properties to pursue when developing DG algorithms~\cite{ahuja2021invariance, arjovsky2019invariant, heinze2018invariant, krueger2021out, peters2016causal, ye2021towards}. 
Consequently, an effective DG measure should comprehensively account for both predictive power and invariance. Neglecting one of these properties can lead to unfavorable outcomes, as exemplified explicitly through the following two instances: 
1) First instance involves a measure placing exclusive emphasis on predictive power. Consider a dataset derived from CMNIST~\cite{arjovsky2019invariant}, where $\mathcal{E}=\{0.90, 0.85, 0.80\}$ are the given environments. The details of CMNIST can be found in Section~\ref{sec:SRCMNIST}. In this case, if the evaluation measure only focuses on predictive power, it is likely that a spurious model, capturing only the spurious correlation, will be chosen. This is because the spurious model can achieve a worst performance of $0.80$, whereas the invariant model, aligned with the invariant correlation, will achieve a performance of $0.75$.
2) Second instance involves a measure placing exclusive emphasis on invariance: We can consider a model with a uniform random output. In this case, the measure only focusing on invariance will likely select the random model because it produces close-to-zero gap over all domains.

We propose worst+gap measure, which encompasses both predictive power and invariance.
This measure can effectively avoid the pitfalls associated with concentrating on only one aspect. 
Upon using the worst+gap measure, a similar observation to the ideal measure arises as shown in Figure~\ref{figure:ERM}(b). 

The theoretical grounds of the worst+gap measure are established in Section~\ref{sec:our_measure}.
The worst+gap measure is derived from the ideal measure through two distinct theorems, each stemming from two simple yet reasonable assumptions. 
While these two assumptions are independent of each other, both theorems lead to a common insight: the essential need to include both worst term and gap term into the DG evaluation measure.

Because our work focuses on investigating measures, we need datasets with sufficiently many environments such that we can evaluate and compare the measures. The datasets used in our study are explained in 
Section~\ref{sec:our_benchmark}.
The experimental outcomes are provided in Section~\ref{sec:experiment}. 
First, we show that the worst+gap measure exhibits superior correlation with the ideal measure in comparison to the average measure. This observation is maintained consistently across a spectrum of 12 diverse configurations of SR-CMNIST (Scale and Ratio controllable CMNIST). 
Second, we repeat the correlation experiments utilizing C-Cats\&Dogs (Colored Cats\&Dogs) and L-CIFAR10 (CIFAR10 with colored Line).
Compared to SR-CMNIST, the datasets contain more realistic images.
Third, we examine the worst+gap measure on PACS-corrupted and VLCS-corrupted datasets. The two datasets are based on PACS and VLCS that are commonly used as DG benchmarks and they serve as real-world datasets in our work.
Finally, we provide discussions on three topics with analysis results in Section~\ref{sec:discussion}.

\section{Preliminaries}
\label{sec:preliminary_and_problem}

\subsection{Notations}
\label{sec:notation}

Let $e_{1},\cdots,e_{n},\cdots,e_{N},e_{N+1},\cdots,e_{N_{\mathcal{E}_{\text{all}}}}$ be a sequence of environments that are i.i.d. drawn from the environment distribution.
Let $\mathcal{E}_{\text{all}}$ be the set of all environments in the above sequence. 
Furthermore, let $\mathcal{E}_{\text{given}}=\{e_{1},\cdots,e_{n},\cdots,e_{N}\}$ be the set of given environments for selecting a learning algorithm and training the model.
Let $\mathcal{D}_{\mathcal{E}_{\text{given}}}=\{D_{e_{1}},\cdots,D_{e_{n}},\cdots,D_{e_{N}}\}$ be the set of given datasets from the set of given environments $\mathcal{E}_{\text{given}}$.
Let $\mathcal{A}$ be the candidate set of DG learning algorithms to be compared.
Let $f_{A}^{\mathcal{E}_{\text{given}}}$ be a model that is trained with a DG algorithm $A$ on the set of given datasets $\mathcal{D}_{\mathcal{E}_{\text{given}}}$.
Let $\mathcal{R}_{e}(f_{A}^{\mathcal{E}_{\text{given}}})$ be the error rate of the model $f_{A}^{\mathcal{E}_{\text{given}}}$ evaluated with $D_{e}$ as the single test dataset.

\subsection{Average measure}
\label{sec:average_measure}

We formally define the average measure.
As the number of given environments is usually small, leave-one-environment-out (LOO) test errors are used in the evaluation procedure. 
We divide the set of given datasets $\mathcal{D}_{\mathcal{E}_{\text{given}}}$ into a single dataset $D_{e_{n}}$ for test and the remaining datasets $\mathcal{D}_{\mathcal{E}_{\text{given}}\backslash e_{n}} = \{D_{e_{1}},\cdots, D_{e_{n-1}}, D_{e_{n+1}}, \cdots, D_{e_{N}}\}$ for training to calculate the LOO test errors.
We calculate the LOO test errors $\{\mathcal{R}_{e_{n}}(f_{A}^{\mathcal{E}_{\text{given}}\backslash e_{n}})\}_{n=1}^{N}$ with a fixed algorithm $A$, where $\mathcal{R}_{e_{n}}(f_{A}^{\mathcal{E}_{\text{given}}\backslash e_{n}})$ is the error rate under the dataset $D_{e_{n}}$ and $f_{A}^{\mathcal{E}_{\text{given}}\backslash e_{n}}$ is the model that is trained with the learning algorithm $A$ on the remaining $N-1$ datasets $\mathcal{D}_{\mathcal{E}_{\text{given}}\backslash e_{n}}$.
The average measure can be calculated using the LOO test errors $\{\mathcal{R}_{e_{n}}(f_{A}^{\mathcal{E}_{\text{given}}\backslash e_{n}})\}_{n=1}^{N}$, as follows:
\begin{equation}
\label{eq:conventional_measure}
\centering
\begin{split}
    \mathcal{R}^{\text{AVG}}_{A} = \frac{1}{N}\sum_{n=1}^{N}\mathcal{R}_{e_{n}}(f_{A}^{\mathcal{E}_{\text{given}}\backslash e_{n}}).
\end{split}
\end{equation}
Previous works have selected the algorithm with the lowest $\mathcal{R}^{\text{AVG}}_{A}$ to find the best-performing algorithm from the candidate algorithm set $\mathcal{A}$, as follows:
\begin{equation}
\label{eq:conventional_measure_argmin}
\centering
\begin{split}
    A^{*}_{\text{AVG}} = \argmin_{A \in \mathcal{A}} \mathcal{R}^{\text{AVG}}_{A},
\end{split}
\end{equation}
where $A^{*}_{\text{AVG}}$ denotes the algorithm with the lowest $\mathcal{R}^{\text{AVG}}_{A}$ among the candidate algorithms.

\subsection{Ideal measure}
\label{sec:ideal_measure}

We define the ideal measure which should be pursued as the ground-truth criterion when investigating DG evaluation measures. There are two fundamental aspects that are required for the ideal measure. First, unlike a practical evaluation measure that has access to only the datasets associated with $\mathcal{E}_{\text{given}}$, the ideal measure should have access to all the datasets associated with $\mathcal{E}_{\text{all}}$ such that the ground-truth performance of Eq.~(\ref{eq:goal}) can be evaluated for any given model $f$. Therefore, the evaluation of the ideal measure is not possible for real-world applications, and it is used only for the purpose of examining the reliability of the practical measures. Second, the ideal measure should be defined using the single model that is trained with the entire given datasets $\mathcal{D}_{\mathcal{E}_{\text{given}}}$, because that is the model to be trained and used in real-world applications upon the selection of the best learning algorithm $A^{*}$. 
Note that leave-one-environment-out procedure trains $N$ models, $\{f_{A}^{\mathcal{E}_{\text{given}}\backslash e_{n}}\}_{n=1}^{N}$, for each learning algorithm $A$. But the $N$ models are used only for selecting the best learning algorithm. The actual model to be used is trained with the entire given datasets $\mathcal{D}_{\mathcal{E}_{\text{given}}}$ using the selected best learning algorithm $A^{*}$.

Based on the two fundamental aspects, we define the ideal measure for a fixed learning algorithm $A$ as follows:
\begin{equation}
\label{eq:ideal_measure}
    \mathcal{R}^{\text{IDEAL}}_{A} = \max\limits_{e \in \mathcal{E}_{\text{all}}}\mathcal{R}_{e}(f_{A}^{\mathcal{E}_{\text{given}}}).
\end{equation}
The ideal measure $\mathcal{R}^{\text{IDEAL}}_{A}$ is almost the same as the DG goal in Eq.~(\ref{eq:goal}), but $f$ in Eq.~(\ref{eq:goal}) is replaced with $f^{\mathcal{E}_{\text{given}}}_{A}$ that is trained with $\mathcal{D}_{\mathcal{E}_{\text{given}}}$.

$\mathcal{R}^{\text{IDEAL}}_{A}$ can be considered as the true performance of an algorithm $A$.
Thus, the true best algorithm can be selected from the candidate algorithm set $\mathcal{A}$ by selecting an algorithm with the lowest $\mathcal{R}^{\text{IDEAL}}_{A}$, as follows:
\begin{equation}
\label{eq:ideal_measure_argmin}
    A^{*}_{\text{IDEAL}} = \argmin_{A \in \mathcal{A}} \mathcal{R}^{\text{IDEAL}}_{A},
\end{equation}
where $A^{*}_{\text{IDEAL}}$ denotes the algorithm with the lowest $\mathcal{R}^{\text{IDEAL}}_{A}$ among the candidate algorithms.
$\mathcal{R}^{\text{IDEAL}}_{A}$ cannot be accessed in real-world applications.
However, we propose five new datasets that allow the ideal measure to be accessible in Section~\ref{sec:our_benchmark} for the purpose of studying practical evaluation measures.

\section{Proposed measure and theoretical grounds}
\label{sec:our_measure}

We first propose a new evaluation measure named the worst+gap measure for DG. Subsequently, we establish two theorems that offer theoretical grounds for the inclusion of both the ``worst" and ``gap" components within the proposed measure.

\subsection{Definition of the worst+gap measure}
\label{sec:define_our_measure}
The worst+gap measure consists of the \textit{worst} error and the \textit{gap} between the worst and best errors. The worst error refers to the largest leave-one-environment-out test error, and the best error refers to the smallest. The definition is provided below.
\begin{equation}
\label{eq:our_measure}
    \begin{split}
    & \mathcal{R}^{\text{W+G}}(f^{\mathcal{E}_{\text{given}}}_A) =  \underbrace{\max\limits_{e_{n} \in \mathcal{E}_{\text{given}}} R_{e_{n}}(f_{A}^{\mathcal{E}_{\text{given}}\backslash e_{n}})}_{\text{Worst error}}
    \\ &+ \frac{1}{N-2} \underbrace{ \left( \max\limits_{e_{n} \in \mathcal{E}_{\text{given}}}R_{e_{n}}(f_{A}^{\mathcal{E}_{\text{given}}\backslash e_{n}})
    - \min\limits_{e_{n} \in \mathcal{E}_{\text{given}}}R_{e_{n}}(f_{A}^{\mathcal{E}_{\text{given}}\backslash e_{n}})
    \right).}_{\text{Gap between the worst error and the best error}}
    \end{split}
\end{equation}
The worst error and the gap can be considered as the predictive power and the environment invariance, respectively.

When the worst+gap measure is used, the algorithm with the lowest worst+gap measure should be selected from the candidate algorithm set $\mathcal{A}$, as follows:
\begin{equation}
\label{eq:wg_measure_argmin}
    A^{*}_{\text{W+G}} = \argmin_{A \in \mathcal{A}} \mathcal{R}^{\text{W+G}}_{A},
\end{equation}
where $A^{*}_{\text{W+G}}$ denotes the algorithm with the lowest $\mathcal{R}^{\text{W+G}}_{A}$ among the candidate algorithms.

\subsection{Theoretical results}
\label{sec:theory}

To establish a theoretical ground, we need to derive a practical measure that can be proven to approximate the true measure $\max\limits_{e \in \mathcal{E}_{\text{all}}} \mathcal{R}_{e}(f_{})$ in Eq.(\ref{eq:goal}) and can be evaluated using only the given datasets $\mathcal{D}_{\mathcal{E}_{\text{given}}}$. 
To derive the worst+gap measure, we have set two goals. First, adopt a simple yet reasonable assumption. Second, derive two different theorems starting from two different assumptions such that a more general insight can be obtained. Theorem~\ref{theorem:main_cheby} assumes a uniform distribution and Theorem~\ref{theorem:main_mcdiar} assumes a decreasing range. Despite the independence in the two assumptions, we will show that both theorems lead to a common insight -- both of the worst term and the gap term are required. 

For the two theorems, we simplify the proof by assuming a fixed model $f$. After deriving the two theorems, we will reconnect the theorems and the worst+gap measure by recovering from the simplification.

\subsection{Useful theorem and lemma}
\label{sec:useful_theorem}

This section delineates the theorem and lemma employed in the demonstration of our principal theorems, Theorem~\ref{theorem:main_cheby} and Theorem~\ref{theorem:main_mcdiar}.

\begin{theorem} (Chebyshev's inequality~\cite{markov1884certain,bienayme1853considerations}
    \label{thm:che_inequality})
    Let $X$ be a random variable with finite expected value $\mu$ and finite non-zero variance $\sigma^2$. Then for any real number $k>0$,
    \begin{equation}
        \text{\emph{Pr}}\bigg[|X-\mu|\geq k\sigma\bigg]\leq\frac{1}{k^2}
    \end{equation}
\end{theorem}

\begin{lemma} 
\label{lemma:expectation_variance_of_maximum_sup}
If $r_{k}$ is i.i.d. drawn from a uniform distribution $U(a, b)$, then  
    \begin{equation}
        \mean\limits_{r_{k} \sim U(a, b)} \bigg[ \max\limits_{1 \leq k \leq K} r_{k} \bigg] = \frac{a + bK}{K + 1},
    \end{equation}
    and
    \begin{equation}
        \var\limits_{r_{k} \sim U(a, b)} \bigg[ \max\limits_{1 \leq k \leq K} r_{k} \bigg] = (b-a)^{2}\frac{K}{(K+2)(K+1)^{2}}.
    \end{equation}
\end{lemma}
\begin{proof}
Let $r_1,\cdots,r_K$ be i.i.d random variables following the uniform distribution $U(a, b)$ satisfying $a>0$. We define the cumulative distribution function (c.d.f) of the random variable $\max\limits_{1 \leq k \leq K} r_{k}$ as follows:
\begin{equation}
\label{appenlem:1}
    \begin{split}
        F(x)
        &= \text{Pr}\bigg[ \max\limits_{1 \leq k \leq K} r_{k} \leq x \bigg].
    \end{split}
\end{equation}
Since $r_1,\cdots,r_K$ are i.i.d random variables, we can derive as follows:
\begin{equation}
\label{appenlem:2}
    \begin{split}
        F(x)
        &= \text{Pr}\big[ r_{k} \leq x \big]^{K} \\
        &= \frac{(x-a)^K}{(b-a)^K}.
    \end{split}
\end{equation}
Then, we can derive the expectation of $\max\limits_{1 \leq k \leq K} r_{k}$ as follows:
\begin{equation}
\label{appenlem:3}
    \begin{split}
        \mean\limits_{r_{k} \sim U(a, b)} \bigg[ \max\limits_{1 \leq k \leq K} r_{k} \bigg]
        &= \int_{a}^{b}xf(x)dx \\
        &= \Big[xF(x)\Big]_{a}^{b} - \int_{a}^{b}F(x)dx \\
        &= b - \int_{a}^{b} \frac{(x-a)^K}{(b-a)^K} dx \\
        &= b - \Big[ \frac{(x-a)^{K+1}}{(K+1) (b-a)^{K}} \Big]^{b}_{a} \\
        &= b - \frac{b-a}{K+1} \\
        &= \frac{Kb + a}{K+1},
    \end{split}
\end{equation}
where $f(x) = \frac{\partial}{\partial x}F(x)$.
To derive the variance of $\max\limits_{1 \leq k \leq K} r_{k}$, we derive the second moment of $\max\limits_{1 \leq k \leq K} r_{k}$ as follows:
\begin{equation}
\label{appenlem:4}
    \begin{split}
        \mean\limits_{r_{k} \sim U(a, b)} \bigg[ \big(\max\limits_{1 \leq k \leq K} r_{k}\big)^{2} \bigg]
        &= \int_{a}^{b}x^{2} f(x)dx \\
        &= \Big[x^{2}F(x)\Big]_{a}^{b} - \int_{a}^{b}2xF(x)dx \\
        &= b^{2} - \Big[2x\int F(x)dx\Big]_{a}^{b} + \int_{a}^{b} 2 \int F(x)dxdx \\
        &= b^{2} - \Big[2x \frac{(x-a)^{K+1}}{(K+1) (b-a)^{K}} \Big]_{a}^{b} + \int_{a}^{b} 2\frac{(x-a)^{K+1}}{(K+1) (b-a)^{K}} dx \\
        &= b^{2} - \frac{2b(b-a)}{K+1} + \Big[2 \frac{(x-a)^{K+2}}{(K+2)(K+1)(b-a)^{K}} \Big]_{a}^{b} \\
        &= b^{2} - \frac{2b(b-a)}{K+1} + \frac{2(b-a)^{2}}{(K+2)(K+1)},
    \end{split}
\end{equation}
where $f(x) = \frac{\partial}{\partial x}F(x)$.
By Eq.~(\ref{appenlem:3}) and Eq.~(\ref{appenlem:4}), we can derive the variance of $\max\limits_{1 \leq k \leq K} r_{k}$ as follows:
\begin{equation}
    \begin{split}
        \var\limits_{r_{k} \sim U(a, b)} \bigg[ \max\limits_{1 \leq k \leq K} r_{k} \bigg]
        &= \mean\limits_{r_{k} \sim U(a, b)} \bigg[ \big(\max\limits_{1 \leq k \leq K} r_{k}\big)^{2} \bigg] - \mean\limits_{r_{k} \sim U(a, b)} \bigg[ \max\limits_{1 \leq k \leq K} r_{k} \bigg]^2 \\
        &= b^{2} - \frac{2b(b-a)}{K+1} + \frac{2(b-a)^{2}}{(K+2)(K+1)} - \bigg(\frac{Kb + a}{K+1}\bigg)^2 \\
        &= (b-a)^{2}\frac{K}{(K+2)(K+1)^{2}}.
    \end{split}
\end{equation}
\end{proof}

\subsubsection{Theorem assuming a uniform distribution}
\label{sec:uniform}
Uniform distribution is one of the most commonly encountered distributions for a bounded random variable. 
With the simple assumption of $\mathcal{R}_{e}(f_{})$ following a uniform distribution, we can derive the following theorem that indicates that the ideal measure relates to the worst error and the gap.

\begin{theorem} 
\label{theorem:main_cheby}
    Suppose that $\mathcal{R}_{e}(f_{})$ 
    is i.i.d. drawn from a uniform distribution $U(a_{\theta}, b_{\theta})$, 
    where $\theta$ are the parameters of $f$.
    Furthermore, suppose that $N_{\mathcal{E}_{\text{all}}} \gg N$, where $N_{\mathcal{E}_{\text{all}}}$ 
    and $N$ are the cardinalities of $\mathcal{E}_{\text{all}}$ and $\mathcal{E}_{\text{given}}$, respectively.
    When $\mathcal{E}_{\text{all}}$ and $\delta \in (0, 1)$ are fixed, with a probability at least $1-\delta^{2}$, 
    \begin{equation}
    \begin{split}
        \max\limits_{e \in \mathcal{E}_{\text{\emph{all}}}} \mathcal{R}_{e}(f_{})
        = \frac{N+1}{N} \max\limits_{e_{n} \in \mathcal{E}_{\text{\emph{given}}}} \mathcal{R}_{e_{n}}(f_{}) + O(\frac{b_{\theta} - a_{\theta}}{N\delta}).
    \end{split}
    \end{equation}
\end{theorem}

\begin{proof}

By Chebyshev's inequality (see Theorem~\ref{thm:che_inequality}), we can consider the following inequality:
\begin{equation}
\label{ieq:thm3.2.1}
\begin{split}
    \text{Pr}\Bigg[ \given[\Big] \max\limits_{e_{i} \in \mathcal{E}_{\text{given}}} \mathcal{R}_{e_{i}}(f_{}) - \mu \given[\Big] \geq t \sigma \Bigg] \leq \frac{1}{t^{2}},
\end{split}
\end{equation}
where $\mu = \mean\big[ \max\limits_{e_{i} \in \mathcal{E}_{\text{given}}} \mathcal{R}_{e_{i}}(f_{}) \big]$, $\sigma = \sqrt{\var\big[ \max\limits_{e_{i} \in \mathcal{E}_{\text{given}}} \mathcal{R}_{e_{i}}(f_{}) \big]}$.
By the complement rule for probability, we can derive the complement of 
Eq.~(\ref{ieq:thm3.2.1}) as follows:
\begin{equation}
\label{ieq:thm3.2.2}
    \text{Pr}\Bigg[ \given[\Big] \max\limits_{e_{i} \in \mathcal{E}_{\text{given}}} \mathcal{R}_{e_{i}}(f_{}) - \mu \given[\Big] \leq t \sigma \Bigg] \geq 1 - \frac{1}{t^{2}}.
\end{equation}
From Lemma~\ref{lemma:expectation_variance_of_maximum_sup},
\begin{equation}
\label{ieq:thm3.2.3}
    \mu = \frac{a_{\theta} + b_{\theta}N}{N + 1},
\end{equation}
and
\begin{equation}
\label{ieq:thm3.2.4}
    \sigma = \frac{b_{\theta}-a_{\theta}}{N+1}\sqrt{\frac{N}{N+2}}.
\end{equation}
Since $\frac{1}{N+1}\sqrt{\frac{N}{N+2}} \leq \frac{1}{N}$, 
\begin{equation}
\label{ieq:thm3.2.5}
    \sigma \leq \frac{b_{\theta}-a_{\theta}}{N}.
\end{equation}
From Eq.~(\ref{ieq:thm3.2.2}), we can derive the following equation by using Eq.~(\ref{ieq:thm3.2.3}) and Eq.~(\ref{ieq:thm3.2.5}):
\begin{equation}
\label{ieq:thm3.2.5_2}
    \begin{split}
        \text{Pr}\Bigg[ \given[\Big] \max\limits_{e_{i} \in \mathcal{E}_{\text{given}}} \mathcal{R}_{e_{i}}(f_{}) - \frac{a_{\theta} + b_{\theta}N}{N + 1} \given[\Big] \leq t \frac{b_{\theta}-a_{\theta}}{N} \Bigg] \geq 1 - \frac{1}{t^{2}}.
    \end{split}
\end{equation}
By choosing $\frac{1}{t}$ as $\delta$, we express Eq.~(\ref{ieq:thm3.2.5_2}) as follows:
\begin{equation}
\label{ieq:thm3.2.6}
\begin{split}
    \text{Pr}\Bigg[ \given[\Big] \max\limits_{e_{i} \in \mathcal{E}_{\text{given}}} \mathcal{R}_{e_{i}}(f_{}) - \frac{a_{\theta} + b_{\theta}N}{N + 1} \given[\Big] \leq \frac{b_{\theta}-a_{\theta}}{N \delta} \Bigg] \geq 1 - \delta^{2}.
\end{split}
\end{equation}
From Eq.~(\ref{ieq:thm3.2.6}), we have determined that with probability at least $1-\delta^{2}$,
\begin{equation}
\label{ieq:thm3.2.7}
    \begin{split}
        \given[\Big] \frac{a_{\theta} + b_{\theta}N}{N + 1} - \max\limits_{e_{i} \in \mathcal{E}_{\text{given}}} \mathcal{R}_{e_{i}}(f_{}) \given[\Big]\leq \frac{b_{\theta}-a_{\theta}}{N \delta}.
    \end{split}
\end{equation}
By multiplying $\frac{N+1}{N}$ in both sides of Eq.~(\ref{ieq:thm3.2.7}), we can derive the the following inequality:
\begin{equation}
\label{ieq:thm3.2.8}
    \given[\Big] \frac{a_{\theta}}{N} + b_{\theta} - \frac{N+1}{N} \max\limits_{e_{i} \in \mathcal{E}_{\text{given}}} \mathcal{R}_{e_{i}}(f_{}) \given[\Big]\leq (b_{\theta}-a_{\theta})\frac{(N+1)}{N^{2} \delta}.
\end{equation}
Since $\given[\Big] \frac{a_{\theta}}{N}+ b_{\theta} - \frac{N+1}{N} \max\limits_{e_{i} \in \mathcal{E}_{\text{given}}} \mathcal{R}_{e_{i}}(f_{}) \given[\Big]$ is asymptotically bounded above by $\frac{1}{N\delta}$, we obtain the following equation: 
\begin{equation}
\label{ieq:thm3.2.9}
    \frac{a_{\theta}}{N} + b_{\theta} - \frac{N+1}{N} \max\limits_{e_{i} \in \mathcal{E}_{\text{given}}} \mathcal{R}_{e_{i}}(f_{}) = O(\frac{b_{\theta}-a_{\theta}}{N\delta}).
\end{equation}
Then,
\begin{equation}
\label{eq:thm3.2.10}
\begin{split}
    b_{\theta} = \frac{N+1}{N} \max\limits_{e_{i} \in \mathcal{E}_{\text{given}}} \mathcal{R}_{e_{i}}(f_{}) + O(\frac{b_{\theta}-a_{\theta}}{N\delta}) - \frac{a_{\theta}}{N}.
\end{split}
\end{equation}
Because it is true that $\frac{a_{\theta}}{N}=O(\frac{b_{\theta}-a_{\theta}}{N})$, we can derive Eq.~(\ref{eq:thm3.2.10}) as follows:
\begin{equation}
\label{eq:thm3.2.11}
    b_{\theta} = \frac{N+1}{N} \max\limits_{e_{i} \in \mathcal{E}_{\text{given}}} \mathcal{R}_{e_{i}}(f_{}) + O(\frac{b_{\theta}-a_{\theta}}{N\delta}).
\end{equation}

Because we assume that $N_{\mathcal{E}_{\text{all}}} \gg N$ and $N_{\mathcal{E}_{\text{all}}}$ is large enough, $\max\limits_{e \in \mathcal{E}_{\text{all}}} \mathcal{R}_{e}(f_{}) = b_{\theta}$.
Thus, we can derive the desired result below with probability at least $1-\delta^{2}$,
\begin{equation}
\begin{split}
    \max\limits_{e \in \mathcal{E}_{\text{all}}} \mathcal{R}_{e}(f_{}) 
    = \frac{N+1}{N} \max\limits_{e_{i} \in \mathcal{E}_{\text{given}}} \mathcal{R}_{e_{i}}(f_{}) + O(\frac{b_{\theta}-a_{\theta}}{N\delta}).
\end{split}
\end{equation}

\end{proof}

According to Theorem~\ref{theorem:main_cheby}, the worst error term is directly derived from the ideal measure.
As for the gap term, it is not directly derived but the big $O$ term contains $b_{\theta} - a_{\theta}$ that can be associated with 
the empirical gap $\max\limits_{e_{n} \in \mathcal{E}_{\text{given}}} \mathcal{R}_{e_{n}}(f_{}) - \min\limits_{e_{n} \in \mathcal{E}_{\text{given}}} \mathcal{R}_{e_{n}}(f_{}) $.

\subsubsection{Theorem assuming a decreasing range with $N$}
\label{sec:range}
It is intuitive to assume that the maximum and minimum of $\{\mathcal{R}_{e_{n}}(f)\}_{n=1}^{N}$ will converge to the supremum and infimum of $\mathcal{R}_{e}(f)$ as $N$ increases.
To develop a theorem related to this intuition, we consider an assumption where the range of the worst error and the range of the best error decrease as $O(\frac{1}{N})$. With the simple assumption, we can derive the following theorem that also indicates that the ideal measure relates to the worst error and the gap.

\begin{theorem}
\label{theorem:main_mcdiar}
    Suppose that the range of $\max\limits_{e_{n} \in \mathcal{E}_{\text{given}}}\mathcal{R}_{e_{n}}(f_{})$ is $[b_{\theta} - \frac{b_{\theta}-a_{\theta}}{N},b_{\theta}]$ and the range of $\min\limits_{e_{n} \in \mathcal{E}_{\text{given}}}\mathcal{R}_{e_{n}}(f_{})$ is $[a_{\theta},a_{\theta} + \frac{b_{\theta}-a_{\theta}}{N}]$, where the range of $\mathcal{R}_{e}(f_{})$ is $[a_{\theta}, b_{\theta}]$, $\theta$ are the parameters of $f_{}$, and $N$ is the cardinality of $\mathcal{E}_{\text{given}}$.
    Furthermore, suppose that $N_{\mathcal{E}_{\text{all}}} \gg N$, where $N_{\mathcal{E}_{\text{all}}}$ 
    and $N$ are the cardinalities of $\mathcal{E}_{\text{all}}$ and $\mathcal{E}_{\text{given}}$, respectively.
    When $\mathcal{E}_{\text{all}}$ is fixed,
    \begin{equation}
    \begin{split}
        &\max\limits_{e \in \mathcal{E}_{\text{\emph{all}}}} \mathcal{R}_{e}(f_{})
        \leq 
        \max\limits_{e_{n} \in \mathcal{E}_{\text{\emph{given}}}} \mathcal{R}_{e_{n}}(f_{}) \\
        &+
        \frac{1}{N-2} \left( \max\limits_{e_{n} \in \mathcal{E}_{\text{given}}} \mathcal{R}_{e_{n}}(f_{}) - \min\limits_{e_{n} \in \mathcal{E}_{\text{given}}} \mathcal{R}_{e_{n}}(f_{}) \right).
    \end{split}
    \end{equation}
\end{theorem}

\begin{proof}
Since we assume that $N_{\mathcal{E}_{\text{all}}}$ is sufficiently larger than $N$, the following is satisfied by the range of $\mathcal{R}_e(f)$.
\begin{equation}
\label{eqn:main_mcdiar1}
    \begin{split}
        & \max\limits_{e \in \mathcal{E}_{\text{all}}} \mathcal{R}_{e}(f_{})=b_{\theta}, \\
        & \min\limits_{e \in \mathcal{E}_{\text{all}}} \mathcal{R}_{e}(f_{})=a_{\theta}.
    \end{split}
\end{equation}
By the assumption that the range of $\max\limits_{e_{n} \in \mathcal{E}_{\text{given}}}\mathcal{R}_{e_{n}}(f_{})$ is $[b_{\theta} - \frac{b_{\theta}-a_{\theta}}{N},b_{\theta}]$, we have the following inequality.
\begin{equation}
\label{eqn:main_mcdiar2}
    \max\limits_{e_{n} \in \mathcal{E}_{\text{given}}}\mathcal{R}_{e_{n}}(f_{}) \geq b_{\theta} - \frac{b_{\theta}-a_{\theta}}{N}.
\end{equation}
By Eq.~(\ref{eqn:main_mcdiar1}) and (\ref{eqn:main_mcdiar2}), we have
\begin{equation}
\label{eqn:main_mcdiar3}
    \max\limits_{e \in \mathcal{E}_{\text{all}}} \mathcal{R}_{e}(f_{})
    \leq
    \max\limits_{e_{n} \in \mathcal{E}_{\text{given}}} \mathcal{R}_{e_{n}}(f_{})
    +
    \frac{1}{N} \bigg(
    \max\limits_{e \in \mathcal{E}_{\text{all}}} \mathcal{R}_{e}(f_{}) - \min\limits_{e \in \mathcal{E}_{\text{all}}} \mathcal{R}_{e}(f_{})
    \bigg).
\end{equation}
Since $\max\limits_{e_{n} \in \mathcal{E}_{\text{given}}}\mathcal{R}_{e_{n}}(f_{}) \in [b_{\theta} - \frac{b_{\theta}-a_{\theta}}{N},b_{\theta}]$ and $\min\limits_{e_{n} \in \mathcal{E}_{\text{given}}}\mathcal{R}_{e_{n}}(f_{}) \in [a_{\theta},a_{\theta} + \frac{b_{\theta}-a_{\theta}}{N}]$, we have the range of $\max\limits_{e_{n} \in \mathcal{E}_{\text{given}}}\mathcal{R}_{e_{n}}(f_{}) - \min\limits_{e_{n} \in \mathcal{E}_{\text{given}}}\mathcal{R}_{e_{n}}(f_{})$ as follows:
\begin{equation}
\label{eqn:main_mcdiar4}
    \Big(\max\limits_{e_{n} \in \mathcal{E}_{\text{given}}}\mathcal{R}_{e_{n}}(f_{}) - \min\limits_{e_{n} \in \mathcal{E}_{\text{given}}}\mathcal{R}_{e_{n}}(f_{})\Big) \in \Big[\frac{N-2}{N}(b_{\theta}-a_{\theta}), b_{\theta}-a_{\theta}\Big].
\end{equation}
By Eq.~(\ref{eqn:main_mcdiar1}) and (\ref{eqn:main_mcdiar4}), we can derive the following inequality:
\begin{equation}
\label{eqn:main_mcdiar5}
    \frac{N-2}{N}\Big(\max\limits_{e \in \mathcal{E}_{\text{all}}} \mathcal{R}_{e}(f_{})-\min\limits_{e \in \mathcal{E}_{\text{all}}} \mathcal{R}_{e}(f_{})\Big) \leq \max\limits_{e_{n} \in \mathcal{E}_{\text{given}}}\mathcal{R}_{e_{n}}(f_{}) - \min\limits_{e_{n} \in \mathcal{E}_{\text{given}}}\mathcal{R}_{e_{n}}(f_{}).
\end{equation}
Then, we have
\begin{equation}
\label{eqn:main_mcdiar6}
    \max\limits_{e \in \mathcal{E}_{\text{all}}} \mathcal{R}_{e}(f_{}) - \min\limits_{e \in \mathcal{E}_{\text{all}}} \mathcal{R}_{e}(f_{}) \leq \frac{N}{N-2} \bigg(\max\limits_{e_{n} \in \mathcal{E}_{\text{given}}} \mathcal{R}_{e_{n}}(f_{}) - \min\limits_{e_{n} \in \mathcal{E}_{\text{given}}} \mathcal{R}_{e_{n}}(f_{})\bigg).
\end{equation}
Therefore, we can derive the following inequality by Eq.~(\ref{eqn:main_mcdiar3}) and (\ref{eqn:main_mcdiar6}).
\begin{equation}
    \begin{split}
        \max\limits_{e \in \mathcal{E}_{\text{all}}} \mathcal{R}_{e}(f_{}) 
        & \leq \max\limits_{e_{n} \in \mathcal{E}_{\text{given}}} \mathcal{R}_{e_{n}}(f_{}) + \frac{1}{N} \bigg(\max\limits_{e \in \mathcal{E}_{\text{all}}} \mathcal{R}_{e}(f_{}) - \min\limits_{e \in \mathcal{E}_{\text{all}}} \mathcal{R}_{e}(f_{})\bigg)\\
        & \leq \max\limits_{e_{n} \in \mathcal{E}_{\text{given}}} \mathcal{R}_{e_{n}}(f_{}) + \Big(\frac{1}{N}\Big)\Big(\frac{N}{N-2}\Big)\bigg(\max\limits_{e_{n} \in \mathcal{E}_{\text{given}}} \mathcal{R}_{e_{n}}(f_{}) - \min\limits_{e_{n} \in \mathcal{E}_{\text{given}}} \mathcal{R}_{e_{n}}(f_{})\bigg)\\
        & = \max\limits_{e_{n} \in \mathcal{E}_{\text{given}}} \mathcal{R}_{e_{n}}(f_{}) + \frac{1}{N-2}\bigg(\max\limits_{e_{n} \in \mathcal{E}_{\text{given}}} \mathcal{R}_{e_{n}}(f_{}) - \min\limits_{e_{n} \in \mathcal{E}_{\text{given}}} \mathcal{R}_{e_{n}}(f_{})\bigg).
    \end{split}
\end{equation}
\end{proof}

Unlike Theorem~\ref{theorem:main_cheby}, Theorem~\ref{theorem:main_mcdiar} directly demonstrates that the ideal measure $\max\limits_{e \in \mathcal{E}_{\text{all}}} \mathcal{R}_{e}(f_{})$ is upper-bounded by the weighted summation of the worst error and the gap. They are accessible measures because they are evaluated with $\mathcal{E}_{\text{given}}$ only.

\subsubsection{Connecting back to the worst+gap measure }
\label{sec:connetion}

For the proofs of Theorem~\ref{theorem:main_cheby} and Theorem~\ref{theorem:main_mcdiar}, we have assumed a fixed model $f$. 
However, the worst+gap measure in Eq.~(\ref{eq:our_measure}) is defined with $f^{\mathcal{E}_{\text{given}}}_{A}$ in the left-hand side and $\{f_{A}^{\mathcal{E}_{\text{given}}\backslash e_{n}}\}_{n=1}^{N}$ in the right-hand side. Therefore, we need to recover Eq.~(\ref{eq:our_measure}) from the theorems. First, we start from Theorem~\ref{theorem:main_mcdiar} because its upper bound is expressed without any ambiguity such as a big $O$ term. 
Second, we note that the theorem applies for a fixed model $f$. Therefore, we can replace $f$ with 
$f_{A}^{\mathcal{E}_{\text{given}}\backslash e_{n}}$ to obtain the following upper bound.
\begin{equation}
\label{eq:bound_incomplete}
\begin{split}
    &\max\limits_{e \in \mathcal{E}_{\text{all}}} \mathcal{R}_{e}(f^{\mathcal{E}_{\text{given}} \backslash e_{n}}_{A})
    \leq \max\limits_{e_{n} \in \mathcal{E}_{\text{given}}} \mathcal{R}_{e_{n}}(f^{\mathcal{E}_{\text{given}} \backslash e_{n}}_{A}) \\
    &+ \frac{1}{N-2}
    \left(
    \max\limits_{e_{n} \in \mathcal{E}_{\text{given}}} \mathcal{R}_{e_{n}}(f_{A}^{\mathcal{E}_{\text{given}} \backslash e_{n}}) - \min\limits_{e_{n} \in \mathcal{E}_{\text{given}}} \mathcal{R}_{e_{n}}(f_{A}^{\mathcal{E}_{\text{given}} \backslash e_{n}})
    \right).
\end{split}
\end{equation}
Note that the right-hand side is exactly the same as in the Eq.~(\ref{eq:our_measure}), but the left-hand side is slightly different because of the term $f_{A}^{\mathcal{E}_{\text{given}}\backslash e_{n}}$.  
Finally, it is reasonable to assume that the performance of a DG model improves when an additional domain is added to $\mathcal{E}_{\text{given}}$. 
That is, we assume $\max\limits_{e \in \mathcal{E}_{\text{all}}} \mathcal{R}_{e}(f^{\mathcal{E}_{\text{given}}}_{A}) \leq \max\limits_{e \in \mathcal{E}_{\text{all}}} \mathcal{R}_{e}(f^{\mathcal{E}_{\text{given}} \backslash e_{n}}_{A})$. Then, the desired bound can be obtained by replacing left-hand-side of Eq.~(\ref{eq:bound_incomplete}) with $\max\limits_{e \in \mathcal{E}_{\text{all}}} \mathcal{R}_{e}(f^{\mathcal{E}_{\text{given}}}_{A})$, as follows:
\begin{equation}
\begin{split}
    &\max\limits_{e \in \mathcal{E}_{\text{all}}} \mathcal{R}_{e}(f^{\mathcal{E}_{\text{given}}}_{A}) 
    \leq \max\limits_{e_{n} \in \mathcal{E}_{\text{given}}} \mathcal{R}_{e_{n}}(f^{\mathcal{E}_{\text{given}} \backslash e_{n}}_{A}) \\
    &+ \frac{1}{N-2} \left(
    \max\limits_{e_{n} \in \mathcal{E}_{\text{given}}} \mathcal{R}_{e_{n}}(f_{A}^{\mathcal{E}_{\text{given}} \backslash e_{n}}) - \min\limits_{e_{n} \in \mathcal{E}_{\text{given}}} \mathcal{R}_{e_{n}}(f_{A}^{\mathcal{E}_{\text{given}} \backslash e_{n}})
    \right).
\end{split}
\end{equation}

We propose the worst+gap measure in Eq.~(\ref{eq:our_measure}) based on the above upper bound derivation of the ideal measure. While the need for an assumption and the adoption of an upper bound remains as limitations, the common appearance of worst and gap terms provides a compelling justification for the proposed worst+gap measure. The practical superiority of the proposed measure will be demonstrated through experiments in Section~\ref{sec:experiment}.

\section{Datasets for studying DG measures}
\label{sec:our_benchmark}

For the purpose of DG measure study, we create five new datasets by slightly modifying five existing datasets.

\subsection{SR-CMNIST dataset}
\label{sec:SRCMNIST}

\begin{figure}[t!]
\begin{center}
\includegraphics[width=0.34\columnwidth]{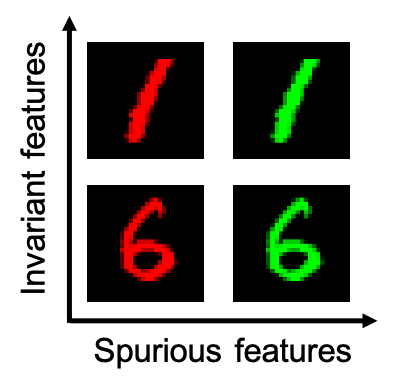}
\end{center}
\caption{
Invariant and spurious features of SR-CMNIST. In SR-CMNIST, color serves as the spurious feature.
\label{fig:dataset_crm}
}
\end{figure}

CMNIST~\cite{arjovsky2019invariant} consists of three datasets for a binary classification task, where the datasets are derived from the MNIST dataset~\cite{deng2012mnist}. Three modifications are made to MNIST to create each dataset of the three environments.
\begin{itemize}
    \item In the first step, a \textit{preliminary binary label} is assigned to each image based on its digit. Digits 0, 1, 2, 3, and 4 are mapped as the first class of CMNIST. Digits 5, 6, 7, 8, and 9 are mapped as the second class of CMNIST. In this way, CMNIST becomes a \textit{binary} classification task.
    \item 
    In the second step, the final class labels of CMNIST are obtained by randomly flipping the preliminary labels with probability 0.25. In this way, a model that only utilizes the desired digit information, i.e., the information associated with the preliminary binary labels in step one, ends up with an error rate of 25\%. 
    \item In the third step, the color labels are applied to the images. For each environment, the color labels are determined by flipping the final labels with probability $e$. Note that $e$ also serves as the notation of the environments in our work. Then, the digits in the images are colored red or green according to the color labels.
\end{itemize}
CMNIST consists of three given environments, $\mathcal{E}_{\text{given}}=\{0.9, 0.8, 0.1\}$, and the three datasets that correspond to the three environments are created by following the procedure described above with $e=0.9, 0.8,$ and $0.1$.

In practice, only $\mathcal{E}_{\text{given}}=\{0.9, 0.8, 0.1\}$ can be used for choosing a learning algorithm and training the associated model. However, the ideal measure $\mathcal{R}^{\text{IDEAL}}_{A}= \max\limits_{e \in \mathcal{E}_{\text{all}}}\mathcal{R}_{e}(f_{A}^{\mathcal{E}_{\text{given}}})$ is defined over all the environments $\mathcal{E}_{\text{all}}$. When $\mathcal{E}_{\text{all}}$ contains both  $e=0.0$ and  $e=1.0$, a model only utilizing the color information has $\mathcal{R}^{\text{IDEAL}}_{A}$ of 100\% because its performance is either 0\% for $e=0.0$ and 100\% for $e=1.0$ or vice versa. On the contrary, $\mathcal{R}^{\text{IDEAL}}_{A}$ is 25\% for a model only utilizing the digit information, as explained in the above second step. 

The study of DG evaluation measures faces a limitation when employing CMNIST as the dataset, primarily owing to the dataset's restricted number of environments.
To overcome this limitation, we increase the number of environments to 101 by choosing $\mathcal{E}_{\text{all}}=\{0.00, 0.01, \cdots, 1.00\}$ while maintaining invariant and spurious features as shown in Figure~\ref{fig:dataset_crm}.
This can be understood as a discrete approximation of the continuously distributed environment $e\sim U(0,1)$.

In addition to the expansion of $\mathcal{E}_{\text{all}}$, we need to decide which environments to include in $\mathcal{E}_{\text{given}}$. Because distributional discrepancy is an important aspect of DG~\cite{arjovsky2019invariant, aubin2021linear, nagarajan2020understanding, hashimoto2018fairness}, we begin by constructing $\mathcal{E}_{\text{given}}$ as the union of two disjoint groups: a majority set $\mathcal{E}_{\text{major}}$ and a minority set $\mathcal{E}_{\text{minor}}$. Note that $\mathcal{E}_{\text{major}}$ and $\mathcal{E}_{\text{minor}}$ represent environment groups divided by the strength range of spurious correlations, and they adhere to the condition $N_\text{major} \geq N_\text{minor}$, where $N_\text{major}$ and $N_\text{minor}$ are the cardinalities of $\mathcal{E}_{\text{major}}$ and $\mathcal{E}_{\text{minor}}$. Finally, we define the control factors \textit{Ratio} and \textit{Scale} as follows:
\begin{itemize}
    \item \textit{Ratio}: the ratio of $N_{\text{major}}$ and $N_{\text{minor}}$ (\ie, $N_{\text{major}}:N_{\text{minor}}$). 
    \item \textit{Scale}: the scale factor of $N$. 
\end{itemize}
For example, if \textit{Ratio} is $4:1$ and \textit{Scale} is $3$, $N_\text{major} = 12$ and $N_\text{minor} = 3$.
We can generate a diverse $\mathcal{E}_{\text{given}}$ by controlling these factors. In our experiments, we consider \textit{Ratio} of  3:1, 4:1, and 5:1, and \textit{Scale} of 1, 2, 3, and 4. By considering their combinations, we end up with a total of 12 different scenarios for $\mathcal{E}_{\text{given}}$. 

The detailed process of environment selection is as follows.
\begin{itemize}
    \item $\mathcal{E}_{\text{major}} = \{0.8, \cdots, 0.8 + 0.1\frac{n_{\text{major}}-1}{N_{\text{major}}-1}, \cdots, 0.9\}$, where $n_{\text{major}}$ denotes the index number. This formation allows a consistent increment within the range of $0.8$ to $0.9$, while forcing the cardinality to be $N_{\text{major}}$.
    \item $\mathcal{E}_{\text{minor}} = \{0.1, \cdots, 0.1 + 0.1\frac{n_{\text{minor}}-1}{N_{\text{minor}}-1}, \cdots, 0.2\}$, where $n_{\text{minor}}$ signifies the index number. This formation allows a consistent increment within the range of $0.1$ to $0.2$, while forcing the cardinality to be $N_{\text{major}}$.
\end{itemize}
Note that we define $\mathcal{E}_{\text{minor}}$ to be $\{0.1\}$ when $N_{\text{minor}}=1$.

\subsection{C-Cats\&Dogs and L-CIFAR10 datasets}
\label{sec:two_add_datasets}

\begin{figure}[t!]
\begin{center}
\subfigure[C-Cats\&Dogs]{
\includegraphics[width=0.34\columnwidth]{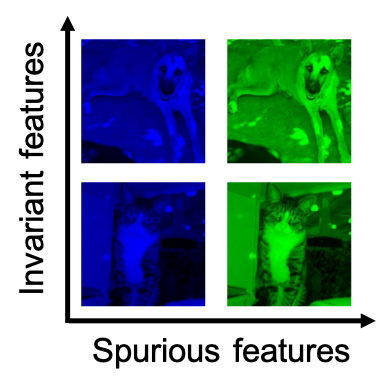}
}
\hspace{0.4cm}
\subfigure[L-CIFAR10]{
\includegraphics[width=0.34\columnwidth]{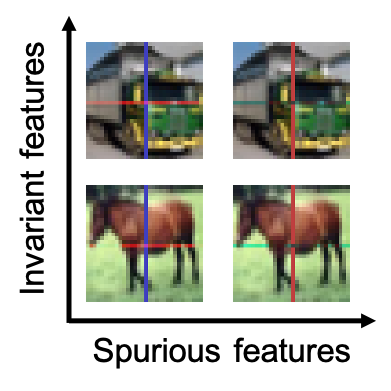}
}
\end{center}
\caption{
Invariant and spurious features of C-Cats\&Dogs and L-CIFAR10 datasets. In C-Cats\&Dogs datasets, color serves as the spurious feature. In L-CIFAR10 dataset, colored line serves as the spurious feature.
\label{fig:dataset}
}
\end{figure}

In addition to SR-CMNIST, we construct and investigate two additional datasets, C-Cats\&Dogs dataset and L-CIFAR10 dataset. 
C-Cats\&Dogs and L-CIFAR10 contain both spurious features that need to be discarded and invariant features that need to be learned as shown in Figure~\ref{fig:dataset}.
Unlike SR-CMNIST, both datasets contain realistic images. Unlike SR-CMNIST and C-Cats\&Dogs, L-CIFAR10 is a 10-class classification dataset. 
We control \textit{Ratio} in C-Cats\&Dogs and L-CIFAR10 for strict DG measure study.
The details of C-Cats\&Dogs and L-CIFAR10 datasets are elaborated in the subsequent sections.

\subsubsection{C-Cats\&Dogs dataset:}
\label{sec:appendix_ccats}
We start from the existing Cats\&Dogs-based dataset~\cite{nagarajan2020understanding, elson2007asirra}. 
While we adhere to the protocol described in the original work, we modify the environment configurations to make it appropriate for DG measure study as outlined below:
\begin{itemize}
    \item We introduce label noise with flipping rate of $25\%$. 
    \item We induce a spurious correlation between color and labels.
    \item We define $\mathcal{E}_{\text{all}}$ to contain 101 environments (datasets), parameterized by a spurious feature (between $0.00$ and $1.00$ in step size of $0.01$).  
    \item We define $\mathcal{E}_{\text{given}}$ to contain 5 environments (datasets), where the spurious features are chosen to be $0.05,0.10,0.15,0.20$, and $0.50$.
\end{itemize}

\subsubsection{L-CIFAR10 dataset}
\label{sec:appendix_lcifar}
We start from the existing CIFAR10-based dataset~\cite{nagarajan2020understanding}. Similar to the C-Cats\&Dogs dataset, we make a list of modifications as outlined below: 
\begin{itemize}
    \item We introduce label noise with flipping rate of $25\%$.
    \item We induce a spurious correlation between colored lines and labels.
    \item We define $\mathcal{E}_{\text{all}}$ to contain 101 environments (datasets), parameterized by a spurious feature (between $0.00$ and $1.00$ in step size of $0.01$).
    \item We define $\mathcal{E}_{\text{given}}$ to contain 5 environments (datasets), where the spurious features 
    are chosen to be $0.50,0.80,0.85,0.90$, and $0.95$.
\end{itemize}

\subsection{Real-world datasets}
\label{sec:our_benchmark_real-world}

\begin{figure}[t]
  \centering 
\resizebox{1.\columnwidth}{!}{
  \subfigure[PACS]{
    \includegraphics[height=0.3\columnwidth]{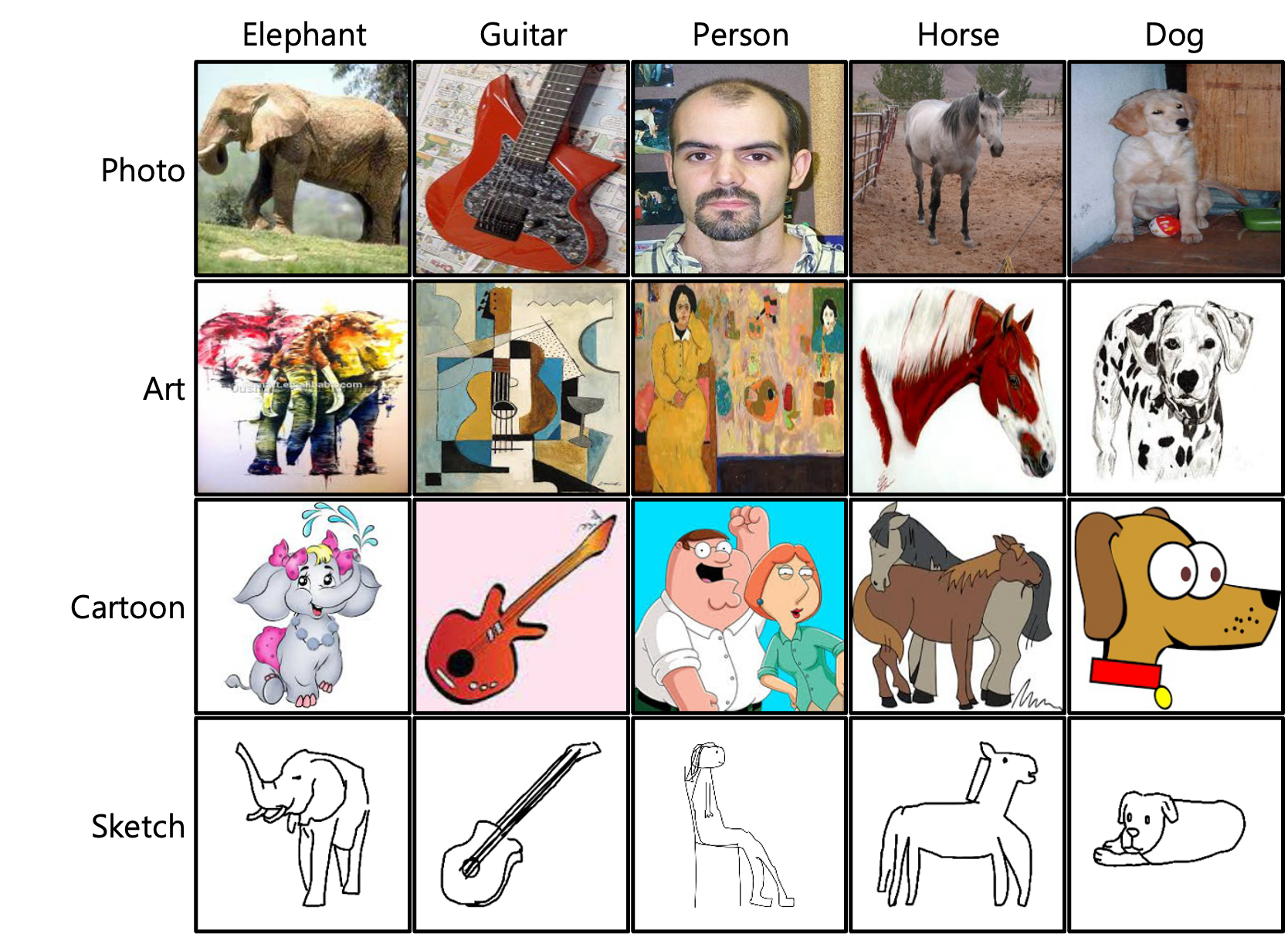}
  }
  \subfigure[Corruptions]{
    \includegraphics[height=0.3\columnwidth]{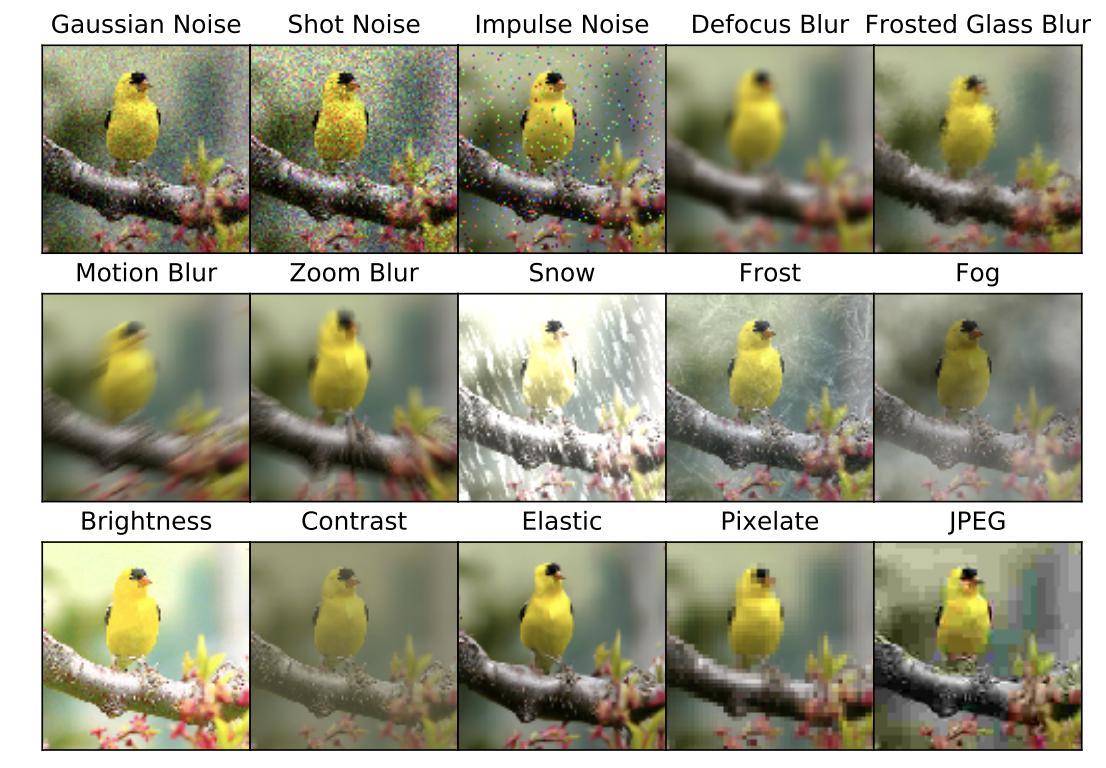}
  }
}
\caption{\label{figure:new_dataset} 
Our PACS-corrupted dataset consists of 64 environments, created by applying 15 different corruptions to the original 4 base environments of PACS. In the figures, (a) illustrates the 4 base environments of PACS and (b) shows the 15 corruptions~\citep{hendrycks2019benchmarking} utilized to increase the number of environments. Similarly, VLCS-corrupted dataset is constructed by applying the 15 filters to the VLCS dataset.
}
\end{figure}

We construct and examine two new datasets called PACS-corrupted and VLCS-corrupted. We begin with two commonly used DG datasets, PACS~\cite{li2017deeper} and VLCS~\cite{fang2013unbiased}. While they are realistic and popular DG datasets, each come with only 4 environments. For studying measures, it is necessary to increase the number of environments significantly. To address this problem, we apply the 15 corruptions utilized in \citet{hendrycks2019benchmarking} to the 4 base environments, resulting in a total of 64 environments (60 corrupted environments + 4 clean environments) in each dataset. We name the new datasets as PACS-corrupted and VLCS-corrupted where each dataset has 64 environments, i.e., $N_{\mathcal{E}_{\text{all}}}=64$. 
See Figure~\ref{figure:new_dataset} for examples of PACS and corruption filters. With the two new datasets, we can sample a reasonable number of given environments $\mathcal{E}_{\text{given}}$, N, from the 64 environments. 
In the experiment section, we first investigate experiment settings of $N=4$ and $N=8$. For the environments (\textit{Envs}) labeled as 1-1-1-1, we sample 4 given environments by selecting a single random corruption for each base environment. For the environments (\textit{Envs}) labeled as 2-2-2-2, we sample 8 given environments by selecting two random corruptions for each base environment. We additionally investigate given environments of 1-0-1-1 and 2-0-2-2.

\section{Experiment results}
\label{sec:experiment}

We compare the quality of two evaluation measures, $\mathcal{R}^{\text{W+G}}_{A}$ and $\mathcal{R}^{\text{AVG}}_{A}$, by comparing their correlation with 
$\mathcal{R}^{\text{IDEAL}}_{A}$.
Specifically, we utilize Spearman's $\rho$~\cite{spearman1961proof} and Kendall's $\tau$~\cite{kendall1938new} that have been widely utilized when comparing measures~\cite{jiang2019fantastic, vedantam2021empirical,yao2022improving,salaudeen2022target,eastwood2022probable}.

We follow the protocol of DomainBed~\cite{gulrajani2020search} for all the experiments. 
We conduct experiments using all 14 algorithms provided in DomainBed.
We reuse the default hyperparameter settings provided in the open-source code base of DomainBed.

\subsection{Results for SR-CMNIST}
\label{sec:correlation}

We calculate the two correlations and the results are shown in Table~\ref{tab:corr_metric}.
First of all, it can be observed that $\mathcal{R}^{\text{IDEAL}}_{A}$ has a stronger correlation with $\mathcal{R}^{\text{W+G}}_{A}$ than $\mathcal{R}^{\text{AVG}}_{A}$ for every single scenario. 
Second, $\mathcal{R}^{\text{W+G}}_{A}$ tends to be a robust measure regardless of the choice of scale or ratio. However, $\mathcal{R}^{\text{AVG}}_{A}$ is not as robust. The correlation value of $\mathcal{R}^{\text{AVG}}_{A}$ decreases significantly when \textit{Ratio} is $5:1$. In particular, the correlation becomes negative ($\rho=-0.458$ and $\tau=-0.377$) when \textit{Scale} is 1. The negative correlation means that a DG algorithm with a better average measure will likely perform worse for the unseen test dataset. 
Overall, the results in Table~\ref{tab:corr_metric} clearly indicate that $\mathcal{R}^{\text{W+G}}_{A}$ is preferred to $\mathcal{R}^{\text{AVG}}_{A}$.

\begin{table}[t!]
\begin{center}
\resizebox{1\columnwidth}{!}{
\begin{tabular}{cccccc}
\toprule
\multicolumn{2}{c}{SR-CMNIST}  & \multicolumn{2}{c}{Spearman’s $\rho$}                                                           & \multicolumn{2}{c}{Kendall’s $\tau$}                                                                                                                     \\ \cmidrule(lr){1-2} \cmidrule(l){3-4} \cmidrule(l){5-6} 
\multirow{1}{*}{\textit{Scale}} & \multirow{1}{*}{\textit{Ratio}}       & $\mathcal{R}^{\text{W+G}}_{A}$ (Ours)                     & $\mathcal{R}^{\text{AVG}}_{A}$                      & $\mathcal{R}^{\text{W+G}}_{A}$ (Ours)                     & $\mathcal{R}^{\text{AVG}}_{A}$                                     \\ \cmidrule(lr){1-2} \cmidrule(l){3-6} 
\multirow{3}{*}{1}                               & 3:1                              & \textbf{0.419 \footnotesize $\pm$ 0.174} & 0.166 \footnotesize $\pm$ 0.088  & \textbf{0.305 \footnotesize $\pm$ 0.135} & 0.063 \footnotesize $\pm$ 0.067 \\ 
                                & 4:1                             & \textbf{0.596 \footnotesize $\pm$ 0.234} & 0.299 \footnotesize $\pm$ 0.170  & \textbf{0.480 \footnotesize $\pm$ 0.202} & 0.209 \footnotesize $\pm$ 0.112   \\
                                & 5:1                             & \textbf{0.600 \footnotesize $\pm$ 0.188} & \underline{-0.458 \footnotesize $\pm$ 0.033} & \textbf{0.465 \footnotesize $\pm$ 0.153} & \underline{-0.377 \footnotesize $\pm$ 0.085} \\ \midrule
\multirow{3}{*}{2}                                & 3:1                             & \textbf{0.643 \footnotesize $\pm$ 0.041} & 0.600 \footnotesize $\pm$ 0.205  & \textbf{0.526 \footnotesize $\pm$ 0.055} & 0.490 \footnotesize $\pm$ 0.182   \\
                                & 4:1                             & \textbf{0.758 \footnotesize $\pm$ 0.071} & 0.471 \footnotesize $\pm$ 0.030  & \textbf{0.613 \footnotesize $\pm$ 0.098} & 0.365 \footnotesize $\pm$ 0.039  \\
                                & 5:1                             & \textbf{0.773 \footnotesize $\pm$ 0.099} & 0.315 \footnotesize $\pm$ 0.153  & \textbf{0.626 \footnotesize $\pm$ 0.129} & 0.277 \footnotesize $\pm$ 0.175  \\ \midrule
\multirow{3}{*}{3}                                & 3:1                             & \textbf{0.782 \footnotesize $\pm$ 0.081} & 0.543 \footnotesize $\pm$ 0.104  & \textbf{0.670 \footnotesize $\pm$ 0.082} & 0.430 \footnotesize $\pm$ 0.107   \\
                                & 4:1                             & \textbf{0.752 \footnotesize $\pm$ 0.014} & 0.644 \footnotesize $\pm$ 0.094  & \textbf{0.602 \footnotesize $\pm$ 0.004} & 0.457 \footnotesize $\pm$ 0.104   \\
                                & 5:1                             & \textbf{0.856 \footnotesize $\pm$ 0.081} & 0.391 \footnotesize $\pm$ 0.221  & \textbf{0.742 \footnotesize $\pm$ 0.107} & 0.346 \footnotesize $\pm$ 0.242  \\ \midrule
\multirow{3}{*}{4}                                & 3:1                             & \textbf{0.764 \footnotesize $\pm$ 0.052} & 0.540 \footnotesize $\pm$ 0.056  & \textbf{0.646 \footnotesize $\pm$ 0.033} & 0.429 \footnotesize $\pm$ 0.077 \\
                                & 4:1                             & \textbf{0.905 \footnotesize $\pm$ 0.076} & 0.721 \footnotesize $\pm$ 0.106  & \textbf{0.802 \footnotesize $\pm$ 0.093} & 0.551 \footnotesize $\pm$ 0.114   \\
                                & 5:1                             & \textbf{0.792 \footnotesize $\pm$ 0.081} & 0.235 \footnotesize $\pm$ 0.251  & \textbf{0.646 \footnotesize $\pm$ 0.062} & 0.175 \footnotesize $\pm$ 0.173  \\ \bottomrule
\end{tabular}
}
\end{center}
\caption{
Correlation with the ideal measure $\mathcal{R}^{\text{IDEAL}}_{A}$. The two measures, $\mathcal{R}^{\text{W+G}}_{A}$ and $\mathcal{R}^{\text{AVG}}_{A}$, are calculated for the 14 DomainBed algorithms using the SR-CMNIST dataset. Then, Spearman’s $\rho$ and Kendall’s $\tau$ with the ideal measure are evaluated. 
We report the mean and standard deviation based on three random seeds.
\label{tab:corr_metric}}
\end{table}

As a further analysis, we produce scatter plots for SR-CMNIST. For a given \textit{Scale} value, we investigate 14 DomainBed algorithms, 3 ratios (3:1, 4:1, and 5:1), and 3 random seeds. Therefore, we evaluate 126 raw values for each of the ideal measure, worst+gap measure, and average measure. Using the three sets of raw values, we create scatter plots as shown in Figure~\ref{fig:mot}.

Roughly speaking, the average measure tends to produce a constant-like error rate (i.e., $\mathcal{R}^{\text{AVG}}_{A}$) regardless of the true error rate (i.e., $\mathcal{R}^{\text{IDEAL}}_{A}$).
Clearly, the average measure can lead to incorrect insights and undesirable algorithm selections.  
Unlike the average measure, worst+gap measure (i.e., $\mathcal{R}^{\text{W+G}}_{A}$) tends to vary in proportion to the ideal measure as shown in Figure~\ref{fig:mot}(b), (c), and (d). Therefore, we can expect it to be more reliable for comparing DG algorithms and selecting a desirable one. An exception is the case of \textit{Scale} of one shown in Figure~\ref{fig:mot}(a). For \textit{Scale} of one, worst+gap measure tends to be also off from the ideal line of slope one. Obviously, worst+gap measure also suffers when $N$ is small for the SR-CMNIST dataset. Its effect can be confirmed in Table~\ref{tab:corr_metric} where the correlation values of worst+gap measure are relatively smaller for \textit{Scale} of one when compared to \textit{Scale} of two or larger. Nonetheless, the correlation values of worst+gap measure are larger than the correlation values of average measure even for \textit{Scale} of one, and worst+gap measure successfully selects the best-performing algorithms as can be confirmed from Figure~\ref{figure:best_sup}(a). The vertical spread of red dots in Figure~\ref{fig:mot}(a) is larger than the spread of blue dots, and most likely the larger spread and the relatively higher correlation contributed toward the superior algorithm selection results.

\begin{figure}[ht]
\resizebox{1.0\columnwidth}{!}{
\centering
  \subfigure[Scale: 1]{
    \includegraphics[height=0.35\columnwidth]{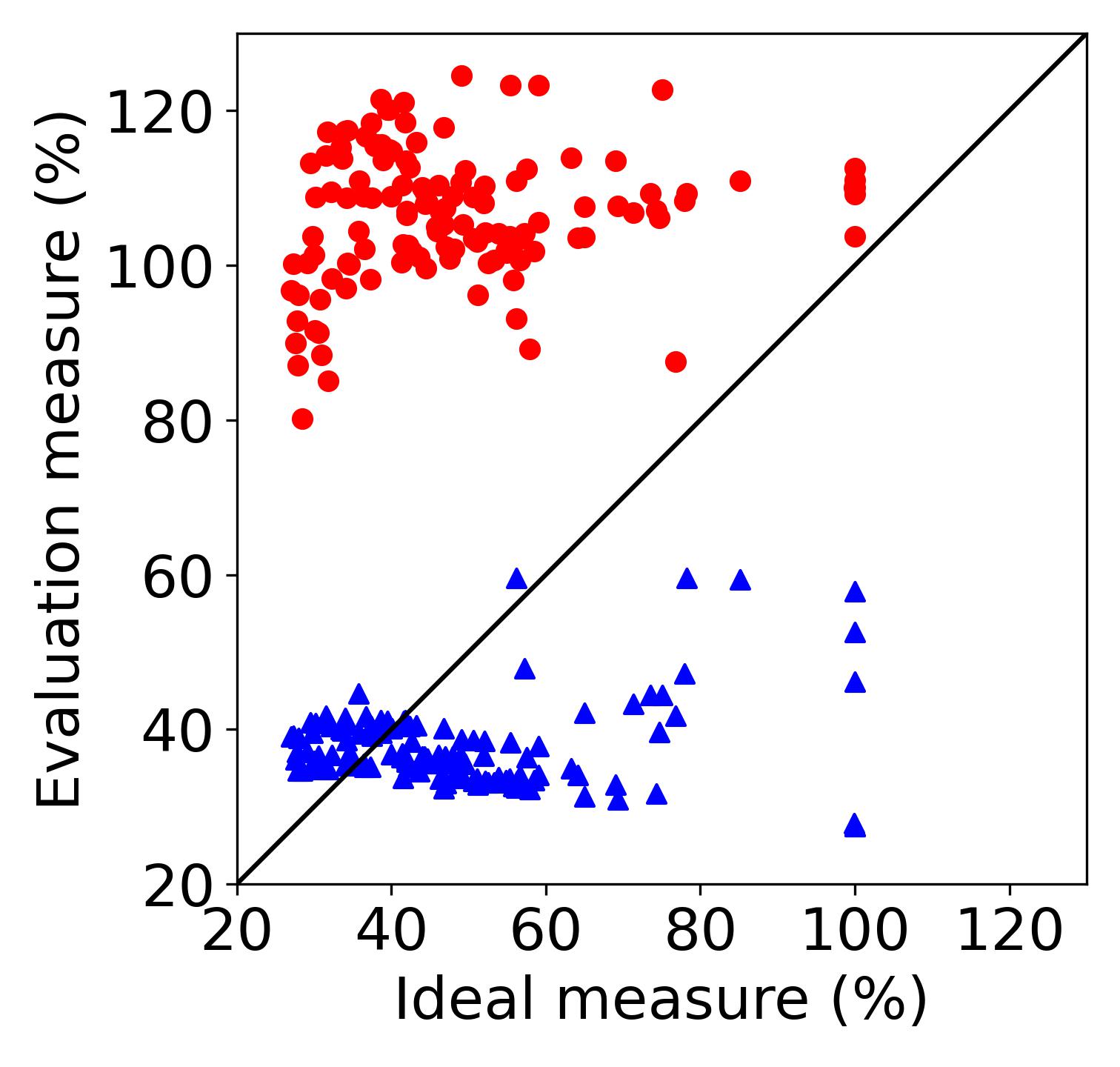}
  }
  \subfigure[Scale: 2]{
    \includegraphics[height=0.35\columnwidth]{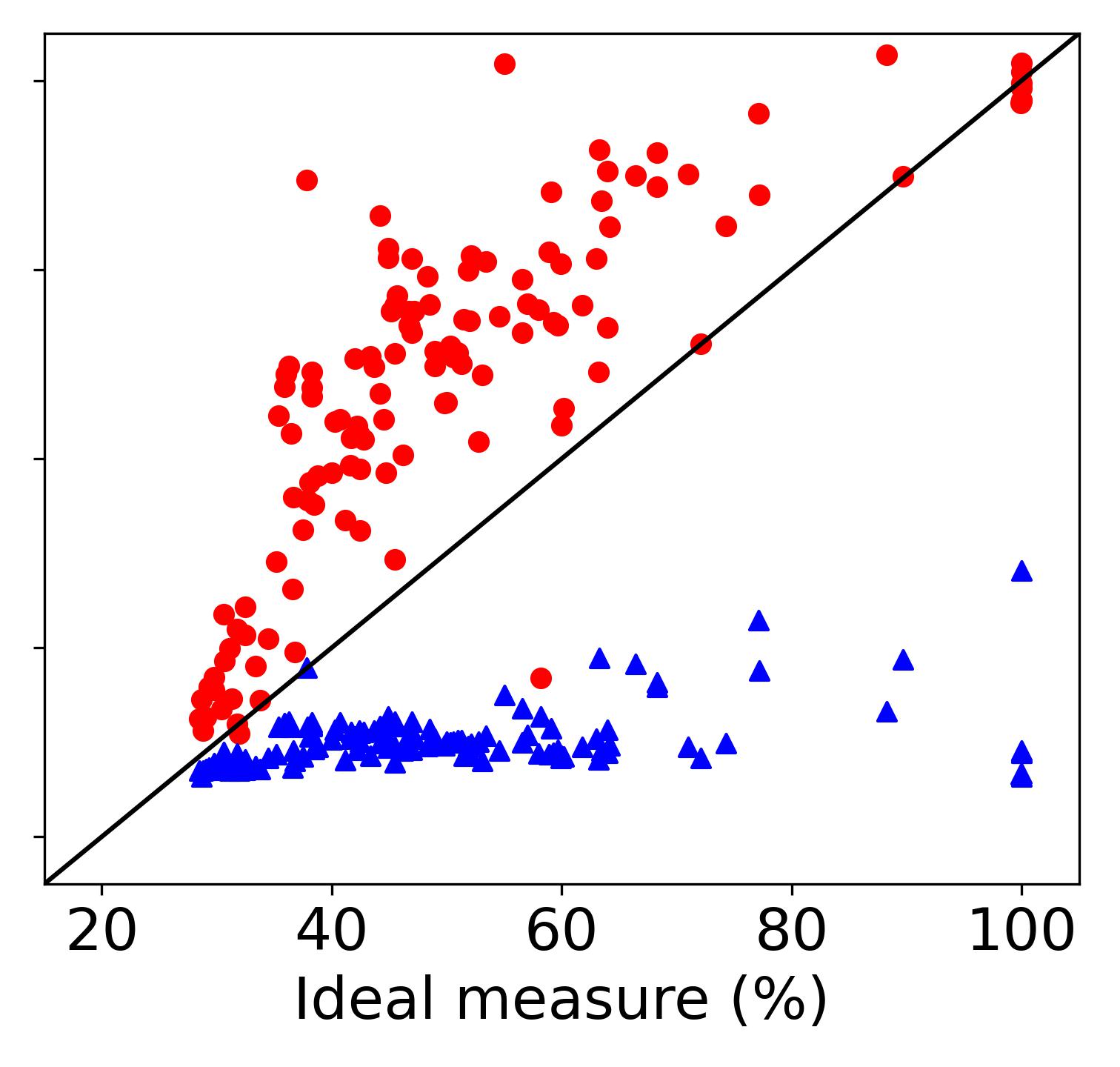}
  }
  \subfigure[Scale: 3]{
    \includegraphics[height=0.35\columnwidth]{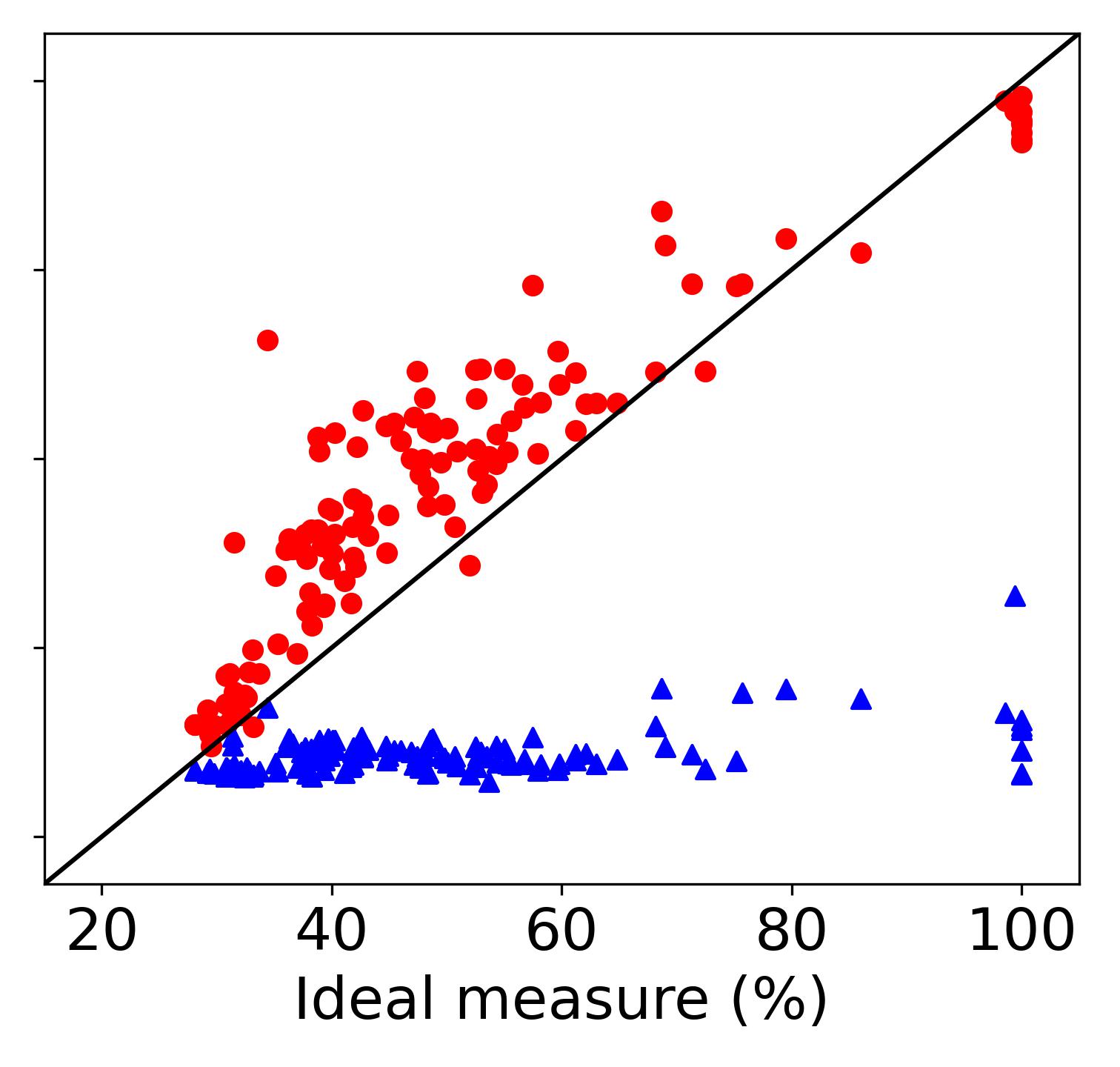}
  }
  \subfigure[Scale: 4]{
    \includegraphics[height=0.35\columnwidth]{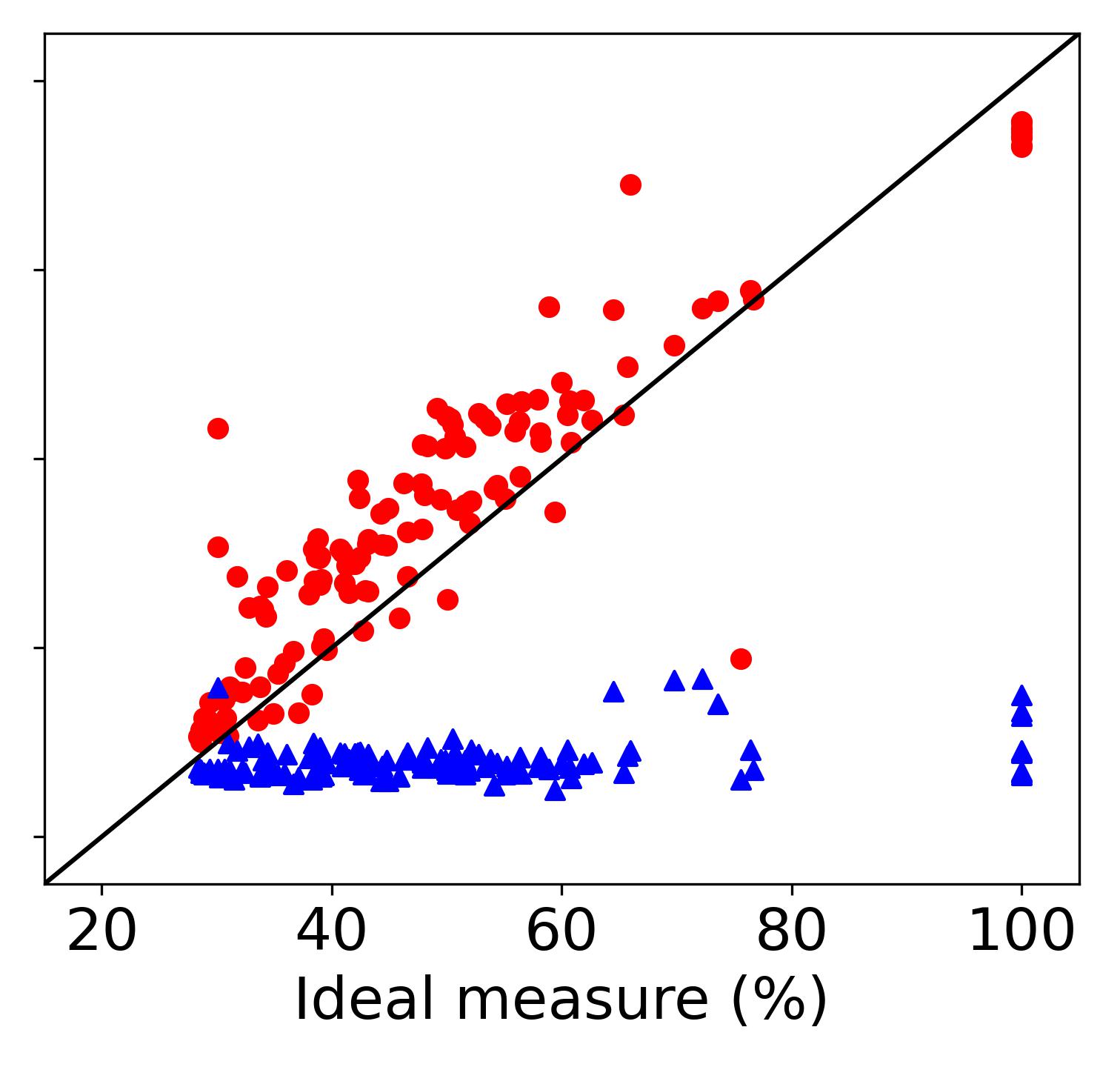}
  }
}
\caption{\label{fig:mot}
Scatter plots of the raw measure values. From the SR-CMNIST experiments, we gather the raw evaluation values for the ideal measure, worst+gap measure, and average measure. The scatter plots are produced using the three sets where the red dots are for `ideal measure vs. worst+gap measure' and the blue dots are for `ideal measure vs. average measure'. 
}
\end{figure}

\subsection{Results for C-Cats\&Dogs and L-CIFAR10 datasets}

While SR-CMNIST provides the advantage of twelve different combinations of \textit{Scale} and \textit{Ratio} for investigation, it is also limited in that it contains simple MNIST images. To overcome the limitation, we have created two additional datasets starting from two existing datasets~\cite{nagarajan2020understanding}. C-Cats\&Dogs dataset is relatively more realistic because it contains real-world images. Its task, however, is still a binary classification. L-CIFAR10 dataset contains real-word images as in C-Cats\&Dogs dataset. Furthermore, its 10-class classification is relatively more complex.

The result of correlation study is shown in Table~\ref{tab:corr_metric_cifar_catdog}. As in the case of SR-CMNIST, $\mathcal{R}^{\text{W+G}}_{A}$ correlates much better with $\mathcal{R}^{\text{IDEAL}}_{A}$. It can be observed that $\mathcal{R}^{\text{AVG}}_{A}$ even has negative correlations for C-Cats\&Dogs dataset. Overall, the results in Table~\ref{tab:corr_metric_cifar_catdog} are aligned with the findings from the SR-CMNIST dataset.

\begin{table}[t!]
\begin{center}
\resizebox{1\columnwidth}{!}{   
\begin{tabular}{cccccc}
\toprule
       &    & \multicolumn{2}{c}{Spearman’s $\rho$}                           & \multicolumn{2}{c}{Kendall’s $\tau$}                            \\
\cmidrule(l){1-2} \cmidrule(l){3-4} \cmidrule(l){5-6}
Dataset       & Ratio & $\mathcal{R}^{\text{W+G}}_{A}$ (Ours) & $\mathcal{R}^{\text{AVG}}_{A}$ & $\mathcal{R}^{\text{W+G}}_{A}$ (Ours) & $\mathcal{R}^{\text{AVG}}_{A}$ \\
\cmidrule(l){1-2} \cmidrule(l){3-4} \cmidrule(l){5-6}
\multirow{3}{*}{C-Cats\&Dogs} & 3:2   & \textbf{0.496 \footnotesize $\pm$ 0.233}                & 0.223 \footnotesize $\pm$ 0.081                & \textbf{0.370 \footnotesize $\pm$ 0.187}                & 0.158 \footnotesize $\pm$ 0.055                \\
                              & 4:1   & \textbf{0.641 \footnotesize $\pm$ 0.020}                & \underline{-0.549} \footnotesize $\pm$ 0.144               & \textbf{0.458 \footnotesize $\pm$ 0.071}                & \underline{-0.429} \footnotesize $\pm$ 0.134               \\
                              & 5:0   & \textbf{0.253 \footnotesize $\pm$ 0.397}                & \underline{-0.482} \footnotesize $\pm$ 0.162               & \textbf{0.231 \footnotesize $\pm$ 0.324}                & \underline{-0.335} \footnotesize $\pm$ 0.124               \\
\cmidrule(l){1-1} \cmidrule(l){2-2} \cmidrule(l){3-4} \cmidrule(l){5-6}
\multirow{3}{*}{L-CIFAR10}    & 3:2   & \textbf{0.600 \footnotesize $\pm$ 0.177}                & 0.338 \footnotesize $\pm$ 0.074                & \textbf{0.451 \footnotesize $\pm$ 0.174}                & 0.319 \footnotesize $\pm$ 0.101                \\
                              & 4:1   & \textbf{0.632 \footnotesize $\pm$ 0.094}                & 0.226 \footnotesize $\pm$ 0.113                & \textbf{0.470 \footnotesize $\pm$ 0.048}                & 0.191 \footnotesize $\pm$ 0.094                \\
                              & 5:0   & \textbf{0.808 \footnotesize $\pm$ 0.007}                & 0.673 \footnotesize $\pm$ 0.098                & \textbf{0.641 \footnotesize $\pm$ 0.034}                & 0.546 \footnotesize $\pm$ 0.113                \\
\bottomrule
\end{tabular}
}
\caption{
Correlation with the ideal measure $\mathcal{R}^{\text{IDEAL}}_{A}$. The experiment was repeated for C-Cats\&Dogs and L-CIFAR10 datasets. We report the mean and standard deviation based on three random seeds. Similar results as in Table~\ref{tab:corr_metric} are obtained. 
\label{tab:corr_metric_cifar_catdog}
}
\end{center}
\end{table}

\subsection{Results for real-world datasets}

In the preceding experiments, we validated a variety of combinations by adjusting the Ratio and Scale of three different datasets. However, it is imperative to validate the findings on real-world datasets. Therefore, we conduct experiments on PACS-corrupted and VLCS-corrupted that are derived from the widely used DG datasets PACS and VLCS.

The correlation results are shown in Table~\ref{tab:realworld_full}. The experiment was carried out for \textit{Envs} 1-1-1-1 and 2-2-2-2. Similar to the previous experiments, $\mathcal{R}^{\text{W+G}}_{A}$ exhibits a higher correlation with $\mathcal{R}^{\text{IDEAL}}_{A}$ for the real-world datasets. As shown in the table, the results reveal that the worst+gap measure consistently achieves a higher Spearman's $\rho$ correlation across all 12 different seeds. For Kendall's $\tau$, the worst+gap measure achieves a higher correlation for 10 seeds, the same as the average measure for 1 seed, and a lower correlation for 1 seed. 
In addition to the results shown in Table~\ref{tab:realworld_full}, we have also investigated \textit{Envs} of 1-0-1-1 and 2-0-2-2. As shown in Table~\ref{tab:realworld_biased}, the worst+gap measure consistently achieves a higher correlation for all experiment cases with different seeds. 
Overall, $\mathcal{R}^{\text{W+G}}_{A}$ clearly outperforms $\mathcal{R}^{\text{AVG}}_{A}$ for the real-world datasets.

\begin{table}[ht]
\begin{center}
\resizebox{0.7\textwidth}{!}{    
\begin{tabular}{ccccccc}
\toprule
\multirow{2}{*}{\textit{Dataset}}  & \multirow{2}{*}{\textit{Envs}}    & \multirow{2}{*}{\textit{Seed}}       & \multicolumn{2}{c}{Spearman’s $\rho$}                           & \multicolumn{2}{c}{Kendall’s $\tau$}                            \\
\cmidrule(l){4-5} \cmidrule(l){6-7}
       &  &   & $\mathcal{R}^{\text{W+G}}_{A}$ & $\mathcal{R}^{\text{AVG}}_{A}$ & $\mathcal{R}^{\text{W+G}}_{A}$ & $\mathcal{R}^{\text{AVG}}_{A}$ \\
\cmidrule(l){1-1} \cmidrule(l){2-2} \cmidrule(l){3-3} \cmidrule(l){4-5} \cmidrule(l){6-7}
\multirow{6}{*}{PACS-corrupted} & \multirow{3}{*}{1-1-1-1} & 0                    & \textbf{0.420}       & 0.335                & \textbf{0.385}       & 0.256                \\
                 &         & 1                    & \textbf{0.918}       & 0.890                & \textbf{0.744}       & 0.718                \\
                 &         & 2                    & \textbf{0.822}       & 0.780                & \textbf{0.684}       & 0.641                \\
\cmidrule(l){2-7}
                 & \multirow{3}{*}{2-2-2-2} & 0                    & \textbf{0.856}       & 0.835                & \textbf{0.667}       & \textbf{0.667}       \\
                 &         & 1                    & \textbf{0.822}       & 0.681                & \textbf{0.692}       & 0.564                \\
                 &         & 2                    & \textbf{0.617}       & 0.613                & \textbf{0.452}       & 0.426                \\
\midrule
\multirow{6}{*}{VLCS-corrupted} & \multirow{3}{*}{1-1-1-1} & 0                    & \textbf{0.495}       & 0.319                & \textbf{0.333}       & 0.282                \\
                 &         & 1                    & \textbf{0.478}       & 0.456                & 0.359                & \textbf{0.385}       \\
                 &         & 2                    & \textbf{0.772}       & 0.670                & \textbf{0.684}       & 0.410                \\
\cmidrule(l){2-7}
                 & \multirow{3}{*}{2-2-2-2} & 0                    & \textbf{0.907}       & 0.857                & \textbf{0.761}       & 0.692                \\
                 &         & 1                    & \textbf{0.900}       & 0.829                & \textbf{0.753}       & 0.684                \\
                 &         & 2                    & \textbf{0.454}       & 0.427                & \textbf{0.338}       & 0.323                \\
\bottomrule
\end{tabular}
}
\caption{
Correlation with the ideal measure $\mathcal{R}^{\text{IDEAL}}_{A}$. The experiment was repeated for PACS-corrupted and VLCS-corrupted. We report the results for all seeds.
\label{tab:realworld_full}
}
\end{center}
\end{table}

\begin{table}[ht]
\begin{center}
\resizebox{0.7\textwidth}{!}{    
\begin{tabular}{ccccccc}
\toprule
\multirow{2}{*}{\textit{Dataset}}  & \multirow{2}{*}{\textit{Envs}}    & \multirow{2}{*}{\textit{Seed}}       & \multicolumn{2}{c}{Spearman’s $\rho$}                           & \multicolumn{2}{c}{Kendall’s $\tau$}                            \\
\cmidrule(l){4-5} \cmidrule(l){6-7}
       &  &   & $\mathcal{R}^{\text{W+G}}_{A}$ & $\mathcal{R}^{\text{AVG}}_{A}$ & $\mathcal{R}^{\text{W+G}}_{A}$ & $\mathcal{R}^{\text{AVG}}_{A}$ \\
\cmidrule(l){1-1} \cmidrule(l){2-2} \cmidrule(l){3-3} \cmidrule(l){4-5} \cmidrule(l){6-7}
\multirow{6}{*}{PACS-corrupted} & \multirow{3}{*}{1-0-1-1} & 0 & \textbf{0.642} & 0.415 & \textbf{0.503} & 0.323 \\
                                     &                             & 1 & \textbf{0.265} & 0.121 & \textbf{0.179} & 0.077 \\
                                     &                             & 2 & \textbf{0.573} & 0.445 & \textbf{0.436} & 0.333 \\
\cmidrule(l){2-7}
                 & \multirow{3}{*}{2-0-2-2} & 0 & \textbf{0.527} & 0.451 & \textbf{0.385} & 0.359 \\
                                     &                             & 1 & \textbf{0.382} & 0.360 & \textbf{0.348} & 0.245 \\
                                     &                             & 2 & \textbf{0.515} & 0.447 & \textbf{0.458} & 0.301 \\
\bottomrule
\end{tabular}
}
\caption{
Correlation with the ideal measure $\mathcal{R}^{\text{IDEAL}}_{A}$. The experiment was repeated for PACS-corrupted. We report the results for all seeds.
\label{tab:realworld_biased}
}
\end{center}
\end{table}
\section{Discussions}
\label{sec:discussion}

\subsection{Effectiveness in selecting the best-performing algorithm}
\label{sec:best_algorithm}

The practical purpose of an evaluation measure lies in selecting the truly best-performing algorithm. For DG, this can be interpreted as selecting an algorithm $A$ with the lowest error rate of $\mathcal{R}^{\text{IDEAL}}_{A}$. To compare the two measures of interest, we define $|\mathcal{R}^{\text{IDEAL}}_{A^{*}_{\text{IDEAL}}} - \mathcal{R}^{\text{IDEAL}}_{A^{*}_{\text{W+G}}}|$ as the performance degradation of the worst+gap measure and 
$|\mathcal{R}^{\text{IDEAL}}_{A^{*}_{\text{IDEAL}}} - \mathcal{R}^{\text{IDEAL}}_{A^{*}_{\text{AVG}}}|$ as the performance degradation of the average measure.
$A^{*}_{\text{IDEAL}}$, $A^{*}_{\text{W+G}}$, and $A^{*}_{\text{AVG}}$ are the best algorithms selected with respect to the ideal measure, worst+gap measure, and average measure, respectively. $A^{*}_{\text{IDEAL}}$ is the truly best algorithm because it is based on the oracle measure.

Figure~\ref{figure:best_sup} shows the experiment results for SR-CMNIST. As in Table~\ref{tab:corr_metric}, performance degradation was evaluated for a variety of \textit{Scale} and \textit{Ratio} scenarios. For all scenarios, it can be observed that worst+gap measure successfully selects the best-performing algorithm and ends up with almost zero performance degradation. Considering that the correlation results in Table~\ref{tab:corr_metric} was positive but not perfect, the results in Figure~\ref{figure:best_sup} implies that worst+gap measure is effective for algorithm selection despite its imperfect correlation with the ideal measure. Therefore, worst+gap is a practically superior measure for SR-CMNIST. In contrast, the average measure's performance degradation can be as large as around $70\%$. As in Table~\ref{tab:corr_metric}, we can observe that the average measure tends to be less effective when the \textit{Ratio} is 5:1.

For C-Cats\&Dogs dataset, we have also repeated the performance degradation study and have obtained similar results. The performance degradation was $10.1\%$ for the worst+gap measure but $39.3\%$ for the average measure. For L-CIFAR10 dataset, the error rates were generally much smaller but the trend was the same ($2.1\%$ vs. $3.0\%$).

\begin{figure}[th]
  \centering 
  \subfigure[\textit{Scale}=1]{
    \includegraphics[height=0.205\columnwidth]{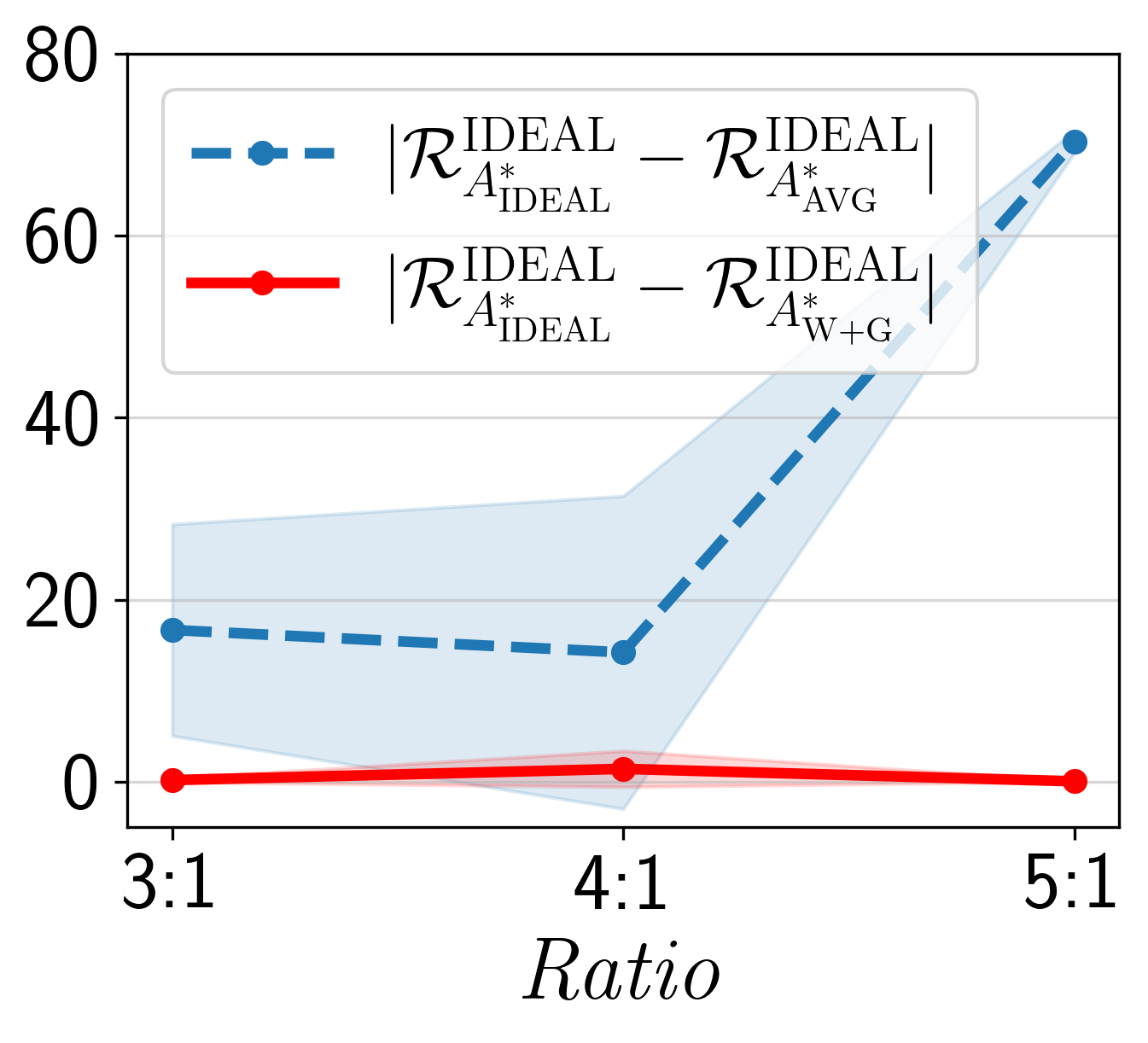}
  }
  \subfigure[\textit{Scale}=2]{
    \includegraphics[height=0.2\columnwidth]{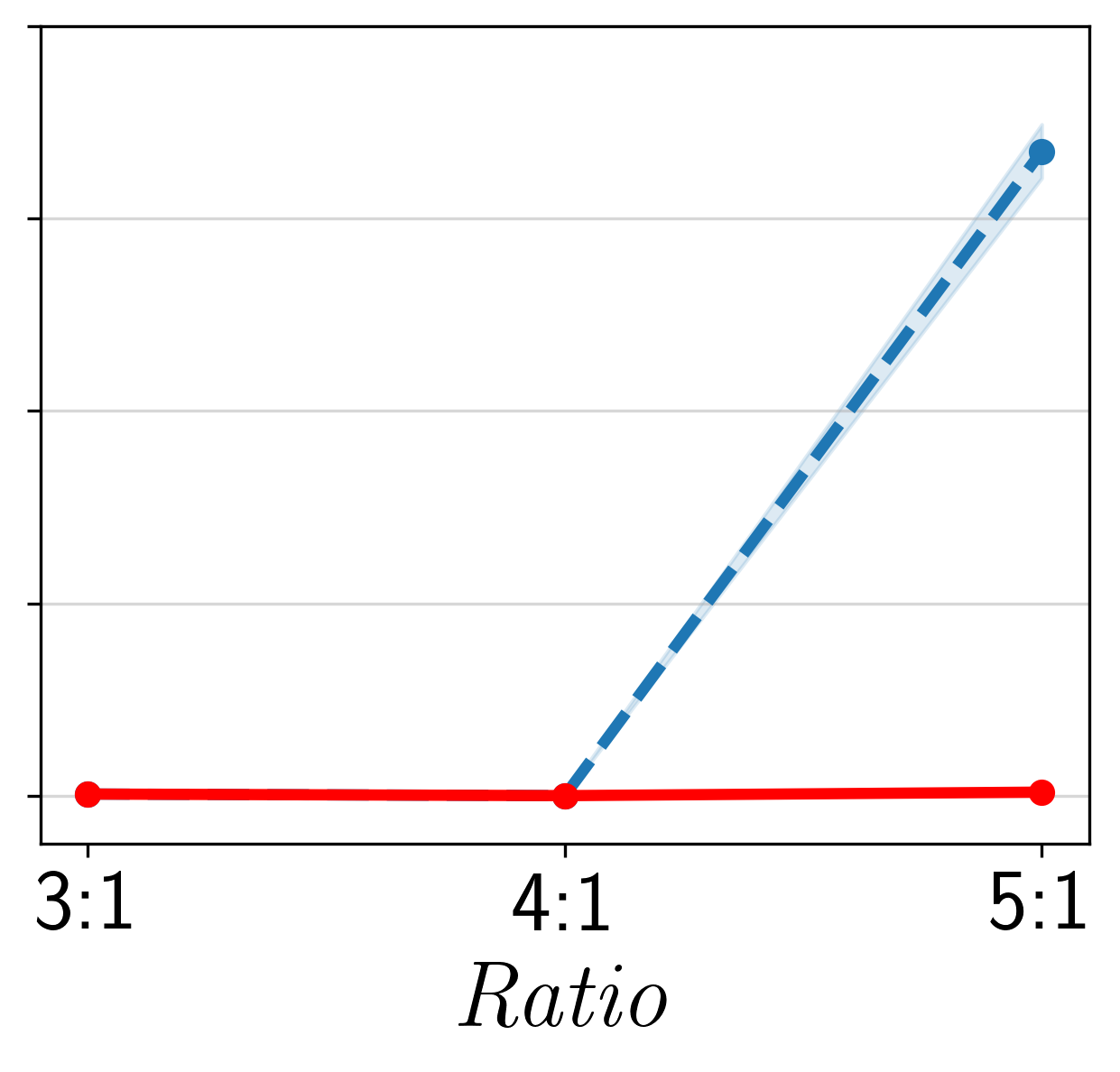}
  }
  \subfigure[\textit{Scale}=3]{
    \includegraphics[height=0.205\columnwidth]{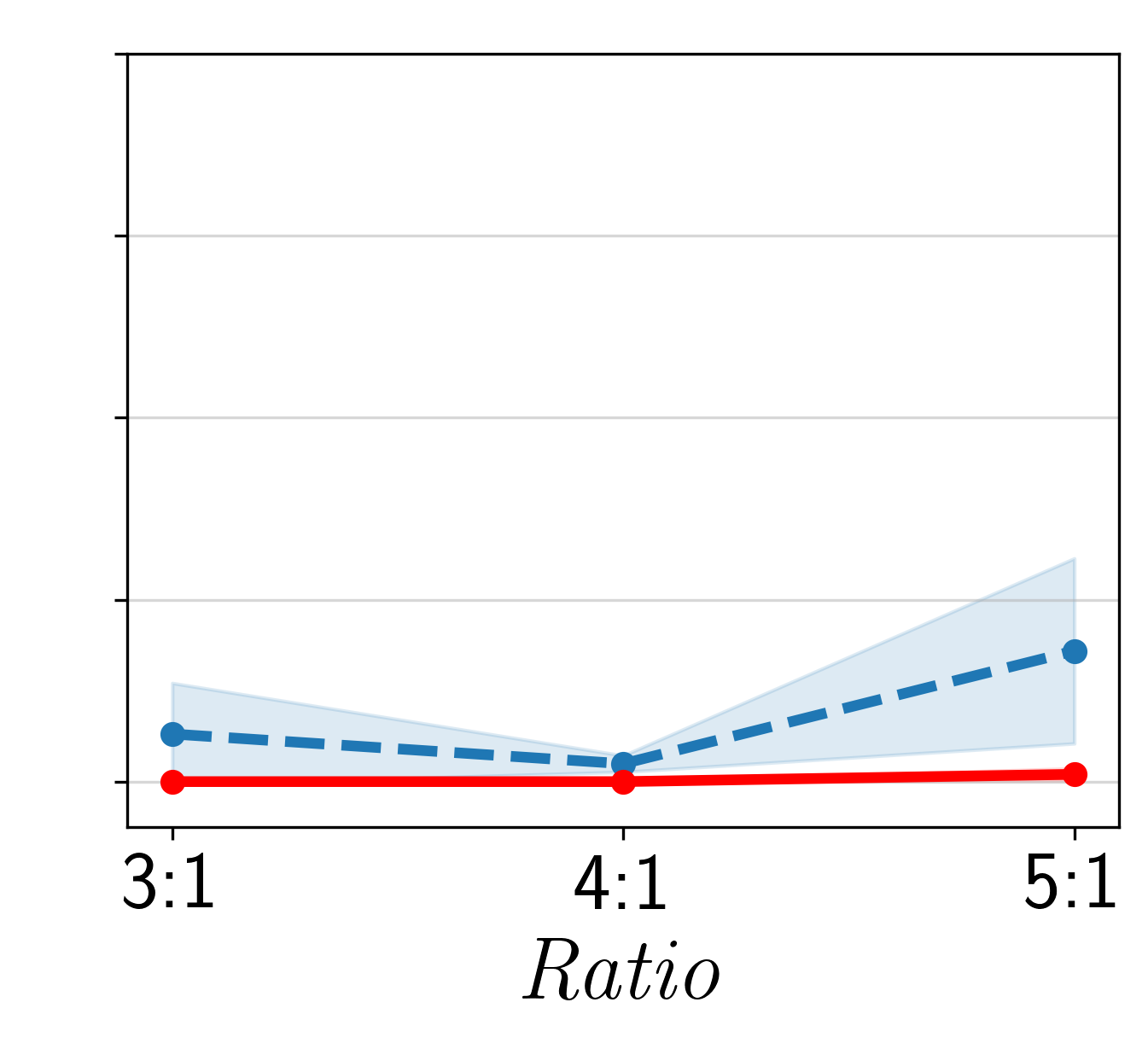}
  }
  \subfigure[\textit{Scale}=4]{
    \includegraphics[height=0.2\columnwidth]{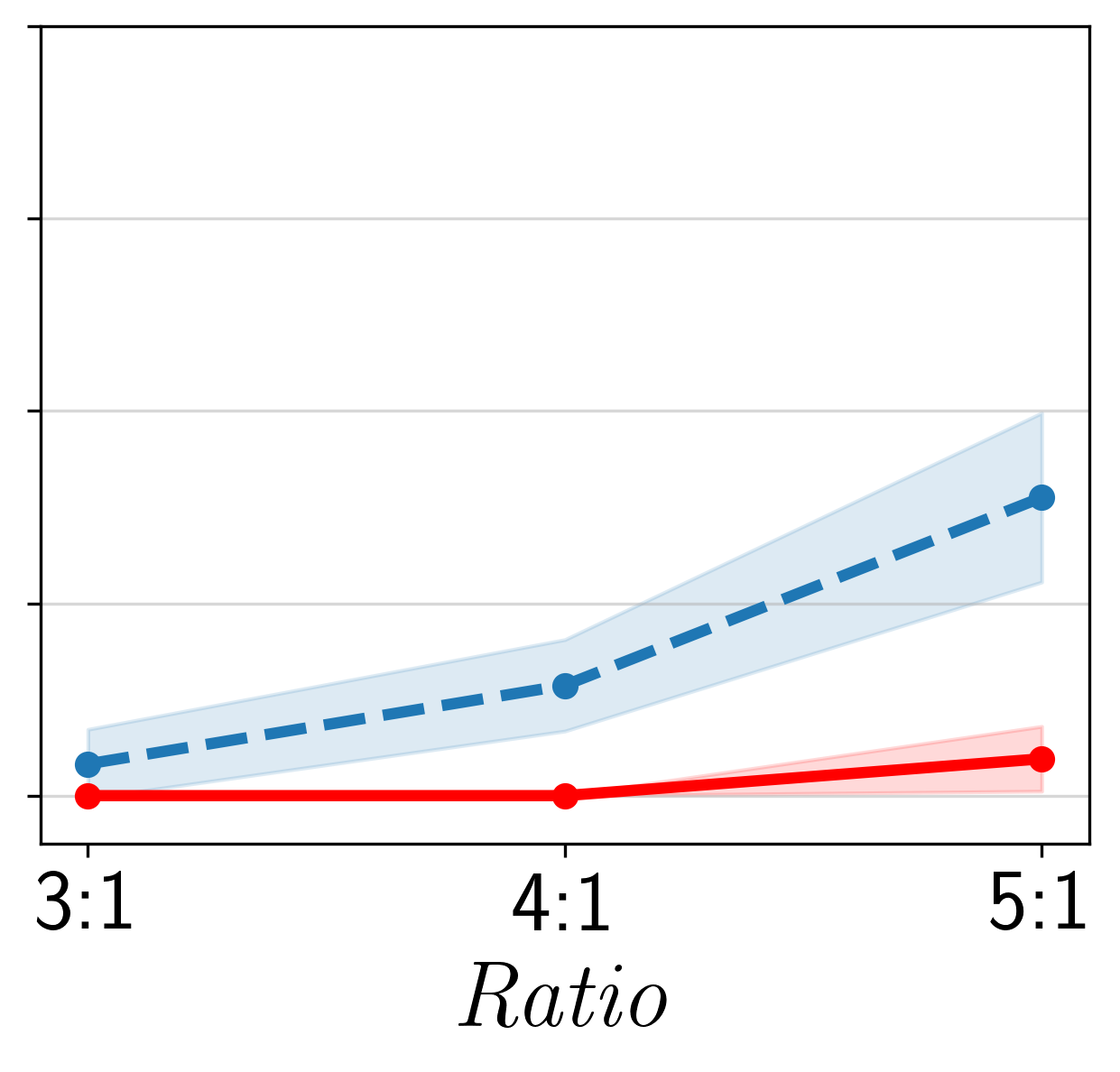}
  }
\caption{\label{figure:best_sup} 
Performance degradation of the selected algorithm with respect to the oracle selection's performance. SR-CMNIST dataset is investigated and the standard deviation is shown as a shaded area. 
}
\end{figure}

\subsection{Worst only, gap only, and worst+gap}
\label{sec:individual}

In Section~\ref{sec:theory}, we have shown that both worst and gap terms emerge under two different assumptions. To confirm if indeed both terms are required for establishing a robust measure, we have compared worst only measure, gap only measure, and worst+gap measure by evaluating  their correlation with the ideal measure $\mathcal{R}^{\text{IDEAL}}_{A}$. 

The results are shown in Table~\ref{tab:individual_spear}. Obviously, none of the three clearly outperforms the other two. However, an interesting observation can be made. For SR-CMNIST, worst only is the best measure and worst+gap is the second best. For C-Cats\&Dogs, gap only is the best measure and again worst+gap is the second best. Finally, worst+gap is the best measure for L-CIFAR10. In this case, the correlations of both worst only and gap only measures are much smaller than worst+gap. 

Overall, worst+gap stays as a decent second best for SR-CMNIST and C-Cats\&Dogs, and it is the best for L-CIFAR10. In contrast, worst only and gap only are much less robust because they can become a clearly inferior measure depending on the characteristics of the dataset. 
As explained earlier, worst only measure can lack invariance and gap only measure can lack predictive capability. 
Hence, it is reasonable to consider that the worst+gap measure is the most robust choice for the experiments outlined in Table~\ref{tab:individual_spear}.

\begin{table}[t!]
\centering
\resizebox{1\columnwidth}{!}{
\begin{tabular}{cccc}
\toprule
Spearman's $\rho$          & Worst only measure          & Gap only measure   & Worst+gap measure          \\
\cmidrule(l){1-1} \cmidrule(l){2-4}
SR-CMNIST                       & \textbf{0.749} & 0.694      & 0.720       \\
C-Cats\&Dogs                     & \underline{0.556}          & \textbf{0.662} & 0.641   \\
L-CIFAR10                 & \underline{0.454}          & \underline{0.460}     & \textbf{0.632}       \\
\midrule
\midrule
Average & 0.586&	0.605&	\textbf{0.664}\\
\bottomrule
\toprule
Kendall's $\tau$          & Worst only measure          & Gap only measure   & Worst+gap measure          \\
\cmidrule(l){1-1} \cmidrule(l){2-4}
SR-CMNIST                       & \textbf{0.621} & 0.573 & 0.594 \\
C-Cats\&Dogs                     & \underline{0.411} & \textbf{0.516} & 0.458 \\
L-CIFAR10                 & \underline{0.337} &\underline{0.349} & \textbf{0.470} \\
\midrule
\midrule
Average & 0.456 & 0.479 & \textbf{0.507} \\
\bottomrule
\end{tabular}
}
\caption{
Comparison of worst+gap with two individual measures. Correlation with the ideal measure $\mathcal{R}^{\text{IDEAL}}_{A}$ is investigated using three benchmark datasets. The best measures are shown in bold and the clearly inferior measures are indicated with underlines.
\label{tab:individual_spear}
}
\end{table}

\subsection{Asymptotic behavior for increasing $N$}
\label{sec:effect_n}

For a well behaving DG measure, it is reasonable to expect the measure to converge to the ideal measure as the number of given environments $N$ increases. 
In Figure~\ref{fig:diff_wg_ideal}, we are showing how $|\mathcal{R}^{\text{IDEAL}}_{A} - \mathcal{R}^{\text{W+G}}_{A}|$ and $|\mathcal{R}^{\text{IDEAL}}_{A} - \mathcal{R}^{\text{AVG}}_{A}|$ vary as $N$ becomes larger for the 14 DomainBed algorithms and SR-CMNIST. \textit{Scale} is used as a surrogate for increasing $N$. Normalized values are plotted with respect to the case of \textit{Scale}$=$1. For all three scenarios, it can be observed that worst+gap measure converges to the ideal measure with increasing $N$. However, the average measure diverges from the ideal measure.

\begin{figure}[t!]
\resizebox{1.0\columnwidth}{!}{
  \centering
  \hspace{-0.2cm}
  \subfigure[\textit{Ratio}$=$3:1]{
    \includegraphics[height=0.41\columnwidth]{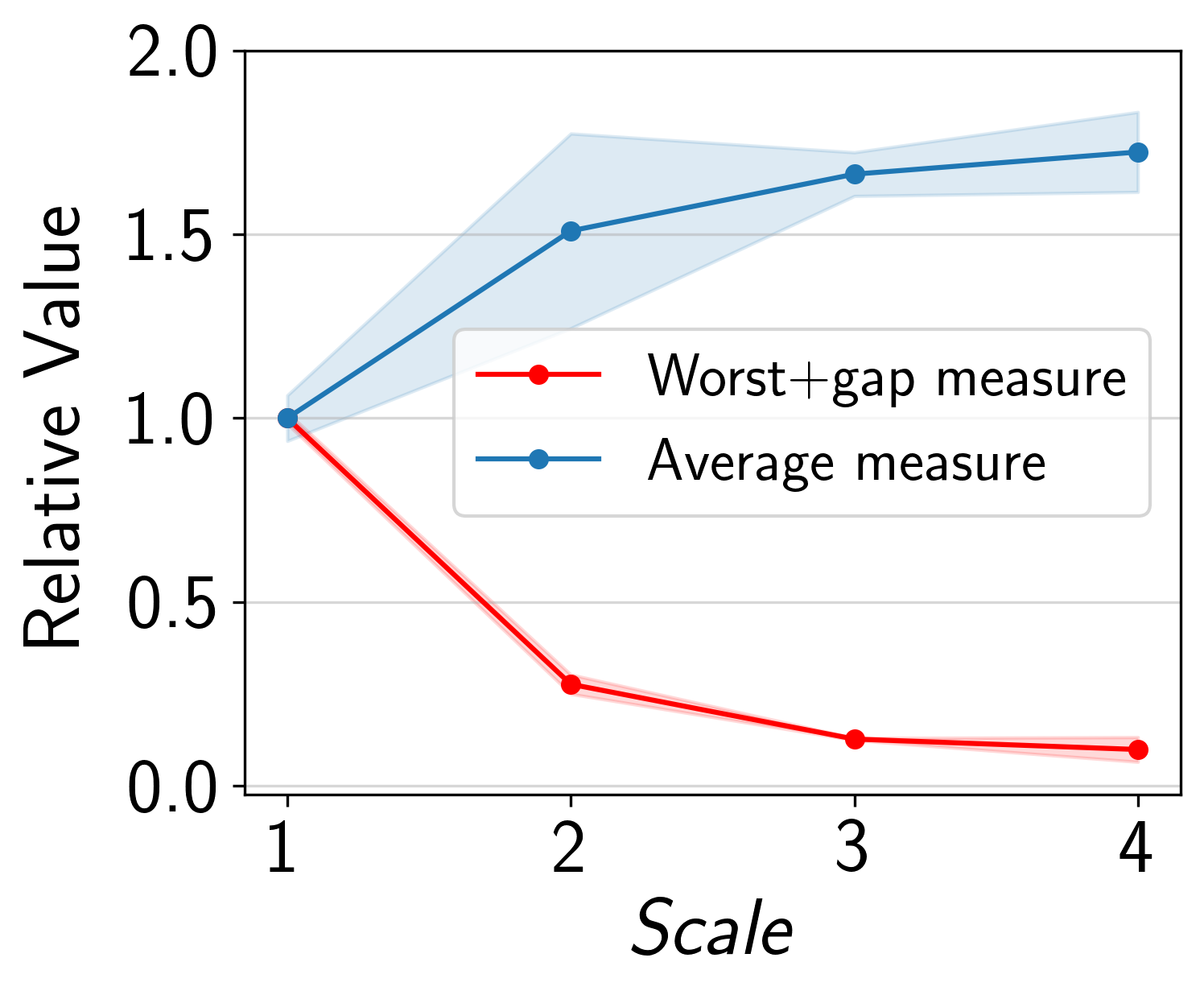}
  }
  \hspace{-0.39cm}
  \subfigure[\textit{Ratio}$=$4:1]{
    \includegraphics[height=0.4\columnwidth]{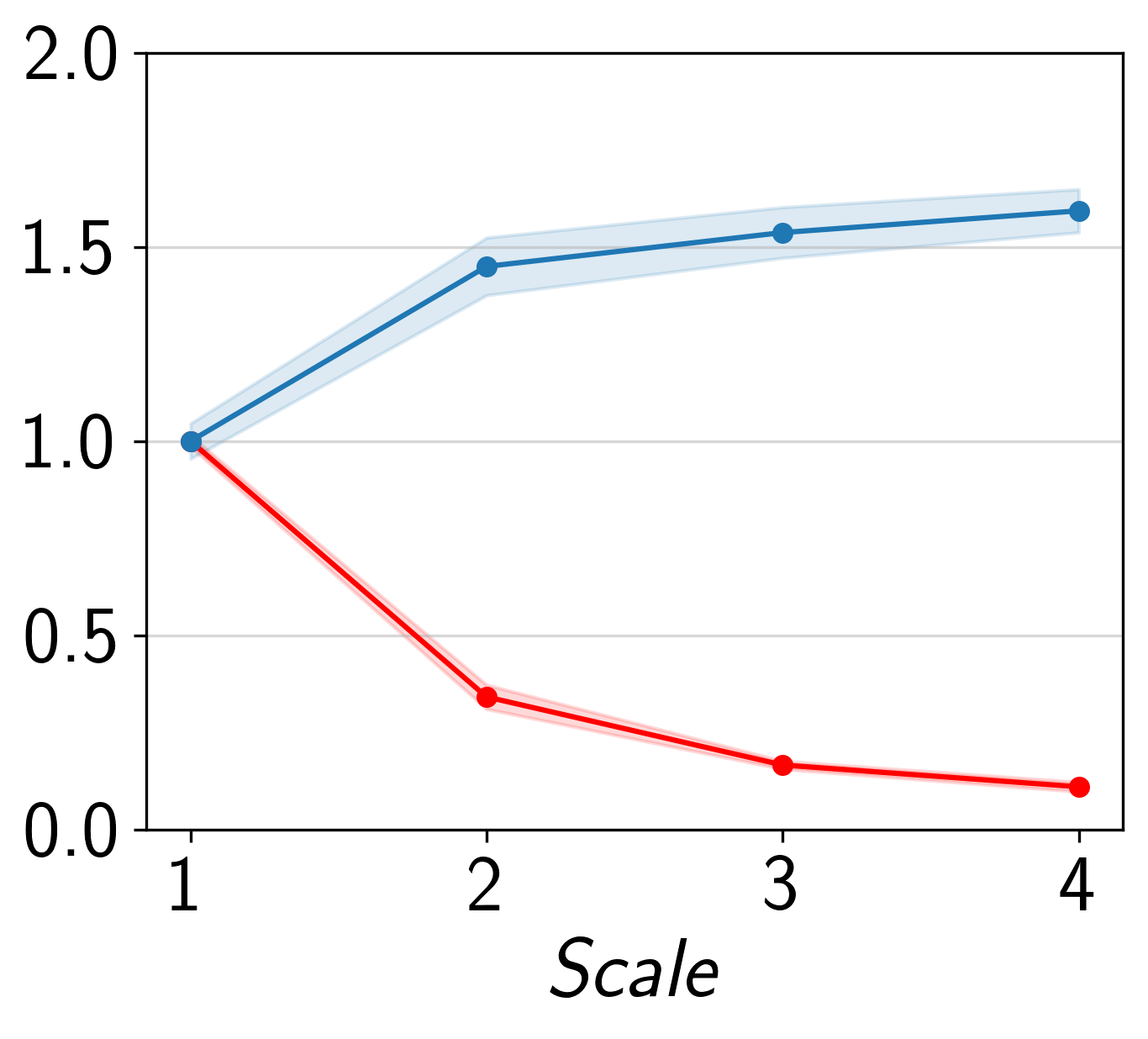}
  }
  \hspace{-0.39cm}
  \subfigure[\textit{Ratio}$=$5:1]{
    \includegraphics[height=0.4\columnwidth]{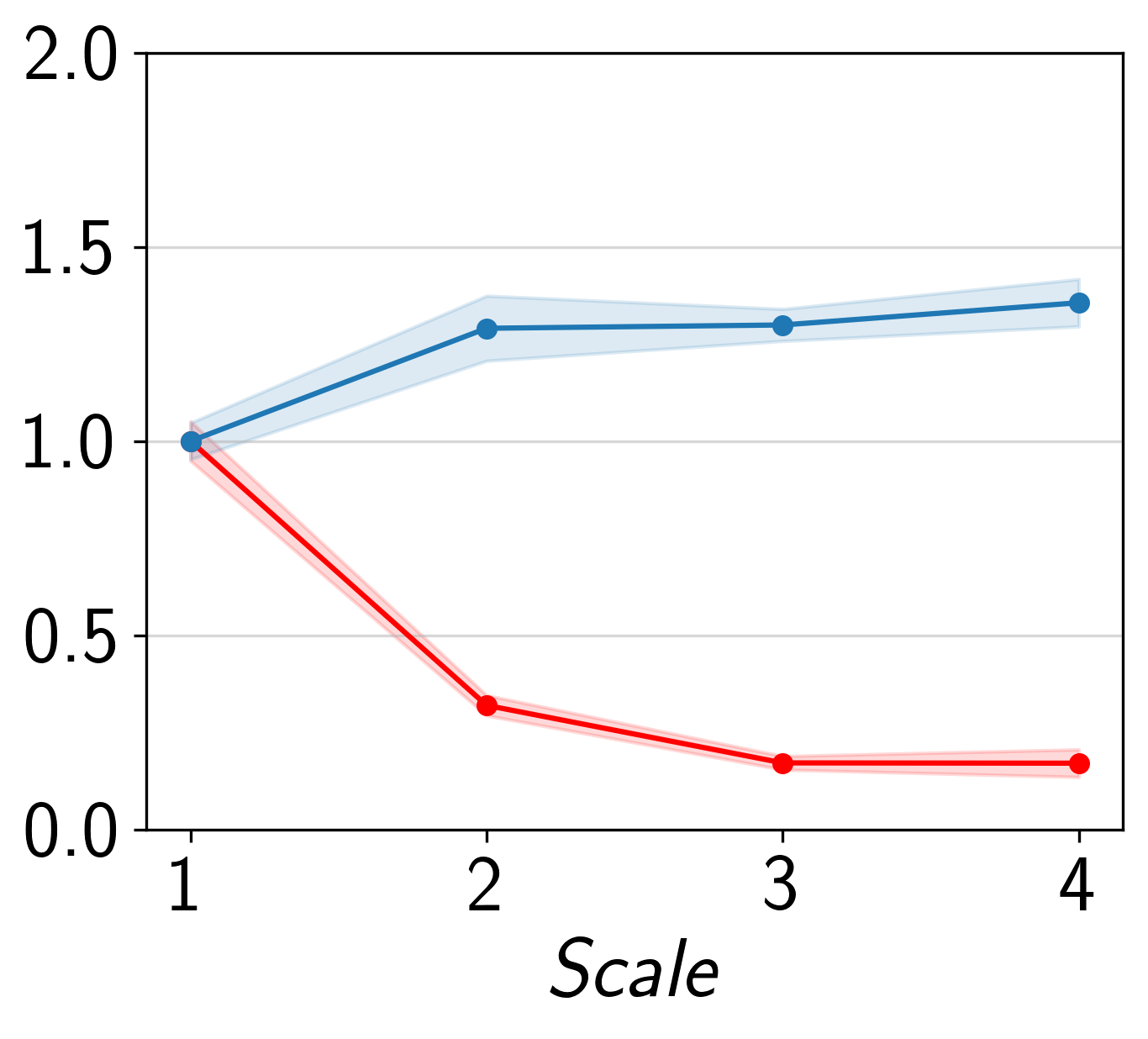}
  }
  \hspace{-0.2cm}
}
\caption{
Asymptotic behavior of $\mathcal{R}^{\text{W+G}}_{A}$ and $\mathcal{R}^{\text{AVG}}_{A}$ as the number of given environments $N$ increases. Normalized values of $|\mathcal{R}^{\text{IDEAL}}_{A} - \mathcal{R}^{\text{W+G}}_{A}|$ and $|\mathcal{R}^{\text{IDEAL}}_{A} - \mathcal{R}^{\text{AVG}}_{A}|$ are plotted for the 14 DomainBed algorithms and SR-CMNIST dataset. Note that $N$ is proportional to \textit{Scale}. Therefore, $N$ increases as \textit{Scale} increases.  
\label{fig:diff_wg_ideal}
}
\end{figure}

\subsection{Failure of ERM}
\label{sec:fulltable_erm}

We compare ERM with the best-performing algorithm $A^{*}_{\text{IDEAL}}$ across various \textit{Scale} on SR-CMNIST. 
In Figures~\ref{c.2.2}, \ref{c.2.3}, and \ref{c.2.4}, when the worst+gap measure is used, similar trends emerge as observed with the ideal measure.
When the average measure is used, however, the performance difference between ERM and $A^{*}_{\text{IDEAL}}$ becomes nearly imperceptible regardless of \textit{Scale} in each figure.

Parallel to the findings in Section~\ref{fig:mot}, an exception arises in the case of \textit{Scale}$=1$ (Figures~\ref{c.2.1}).
For \textit{Scale} of one, the worst+gap measure also diverges from the ideal measure.
Nevertheless, a significant performance difference between ERM and $A^{*}_{\text{IDEAL}}$ persists when $\mathcal{R}^{\text{W+G}}_{A}$ is used.
Furthermore, as discussed in Section~\ref{sec:correlation}, the correlation values of worst+gap measure are larger than the correlation values of average measure even for \textit{Scale} of one, and worst+gap measure successfully selects the best-performing algorithms as can be confirmed from Figure~\ref{figure:best_sup}(a) in the manuscript.

\begin{figure}[ht]
  \centering
  \subfigure[]{
    \includegraphics[width=0.2\columnwidth]{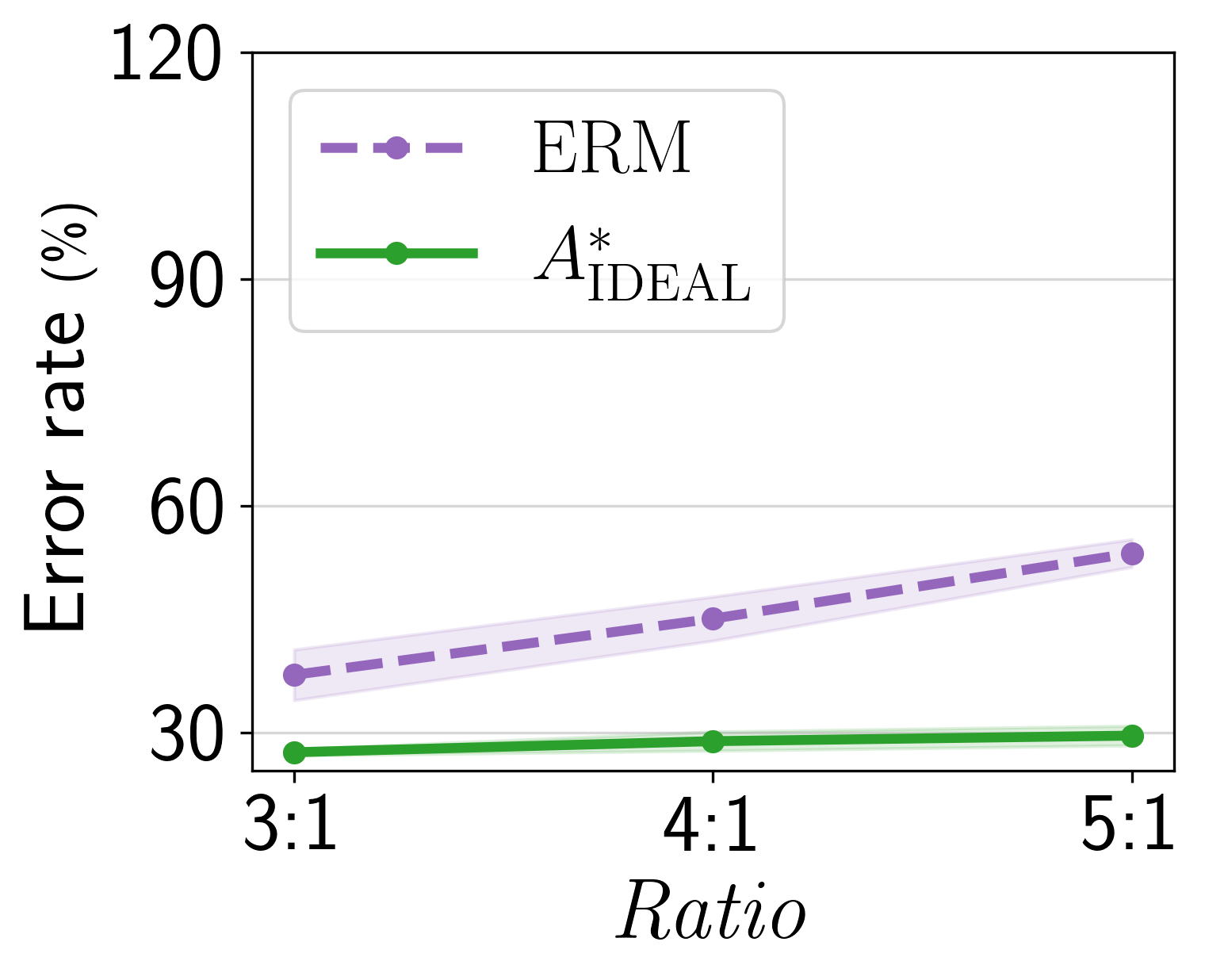}
  }
  \subfigure[]{
    \includegraphics[width=0.2\columnwidth]{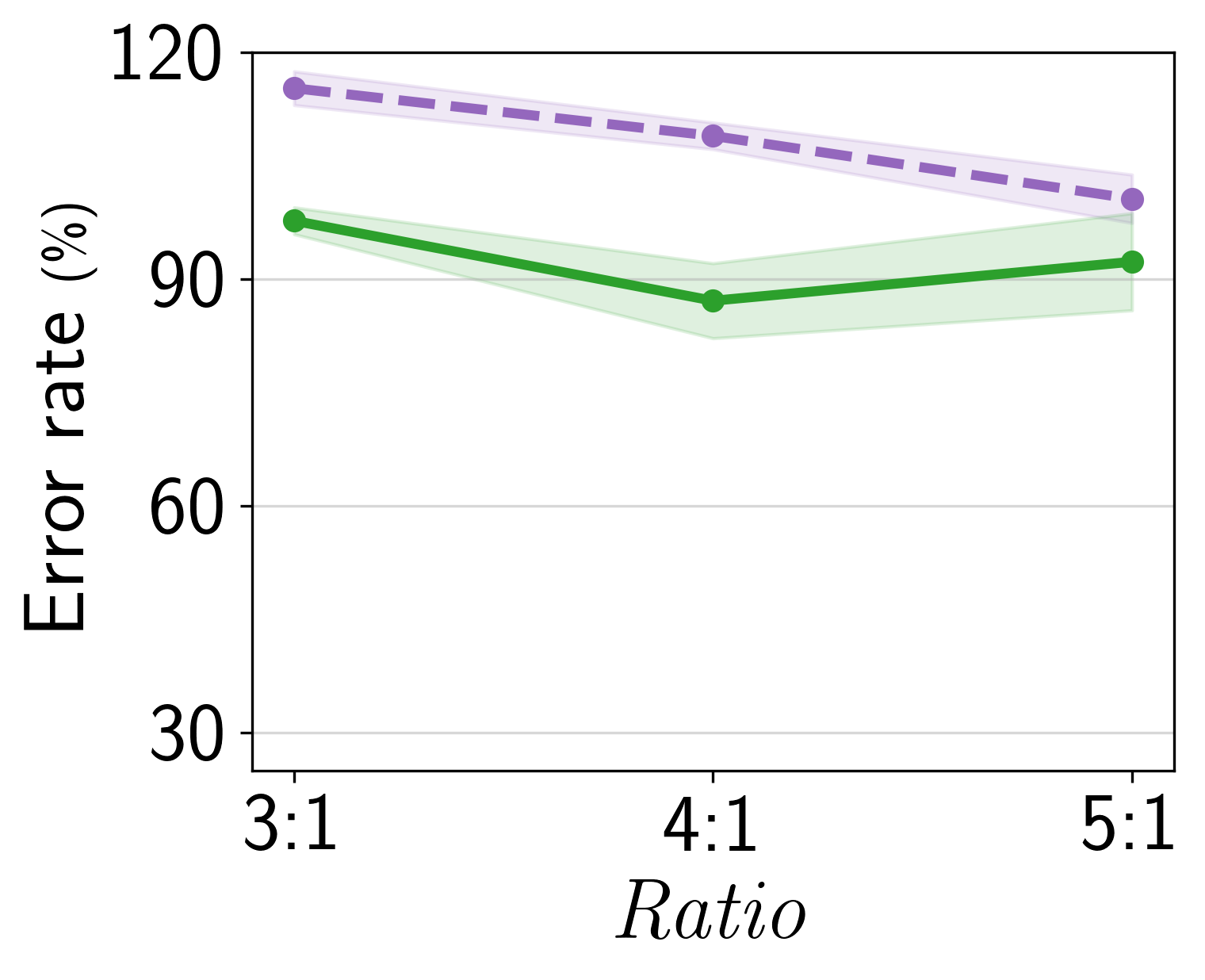}
  }
  \subfigure[]{
    \includegraphics[width=0.2\columnwidth]{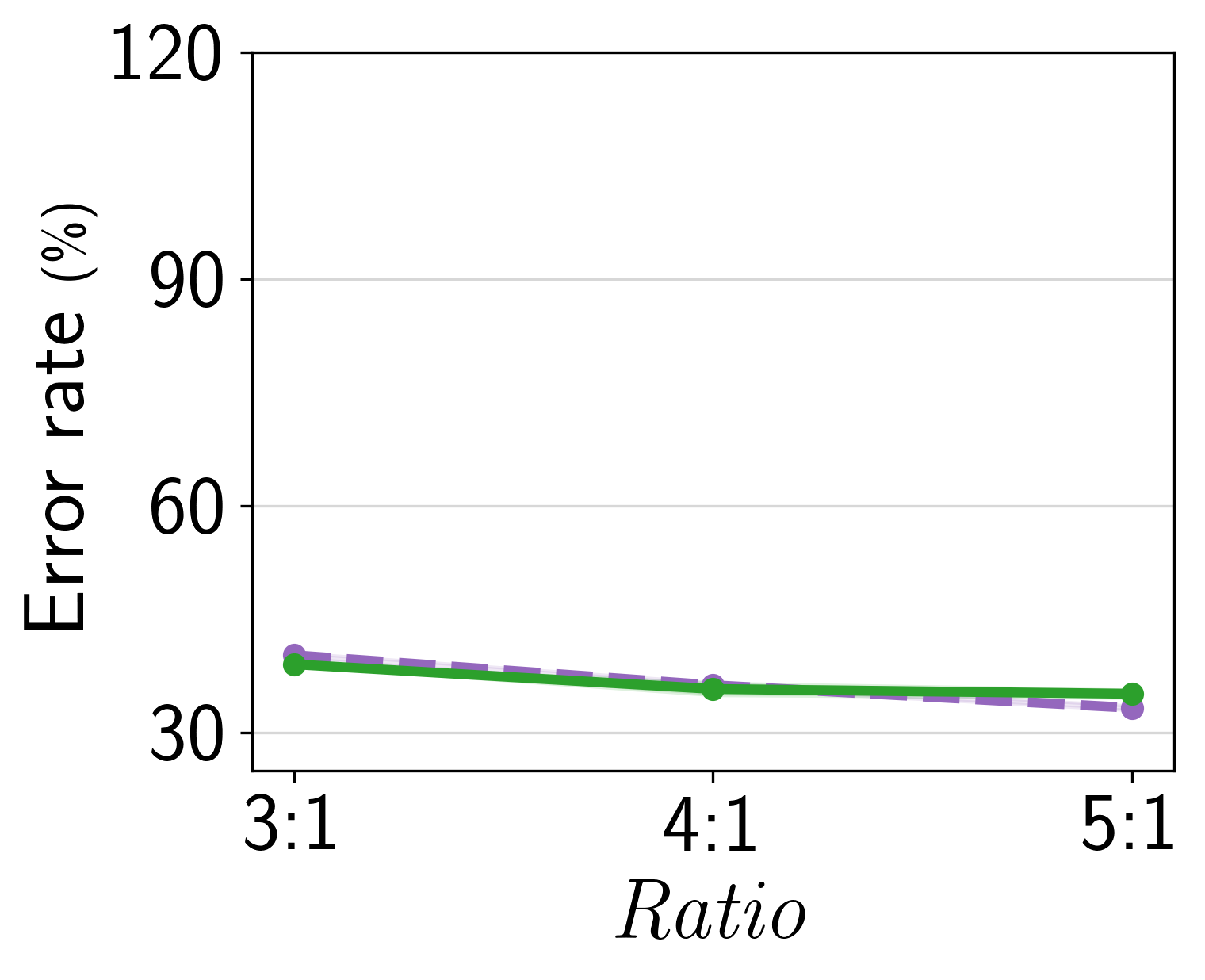}
  }
\caption{
(a) $\mathcal{R}^{\text{IDEAL}}_{A}$, (b) $\mathcal{R}^{\text{W+G}}_{A}$, and (c) $\mathcal{R}^{\text{AVG}}_{A}$ (\%) of ERM and best-performing algorithm $A^{*}_{\text{IDEAL}}$ when \textit{Scale} of SR-CMNIST is 1.
\label{c.2.1}
}
\end{figure}

\begin{figure}[hh]
  \centering
  \subfigure[]{
    \includegraphics[width=0.2\columnwidth]{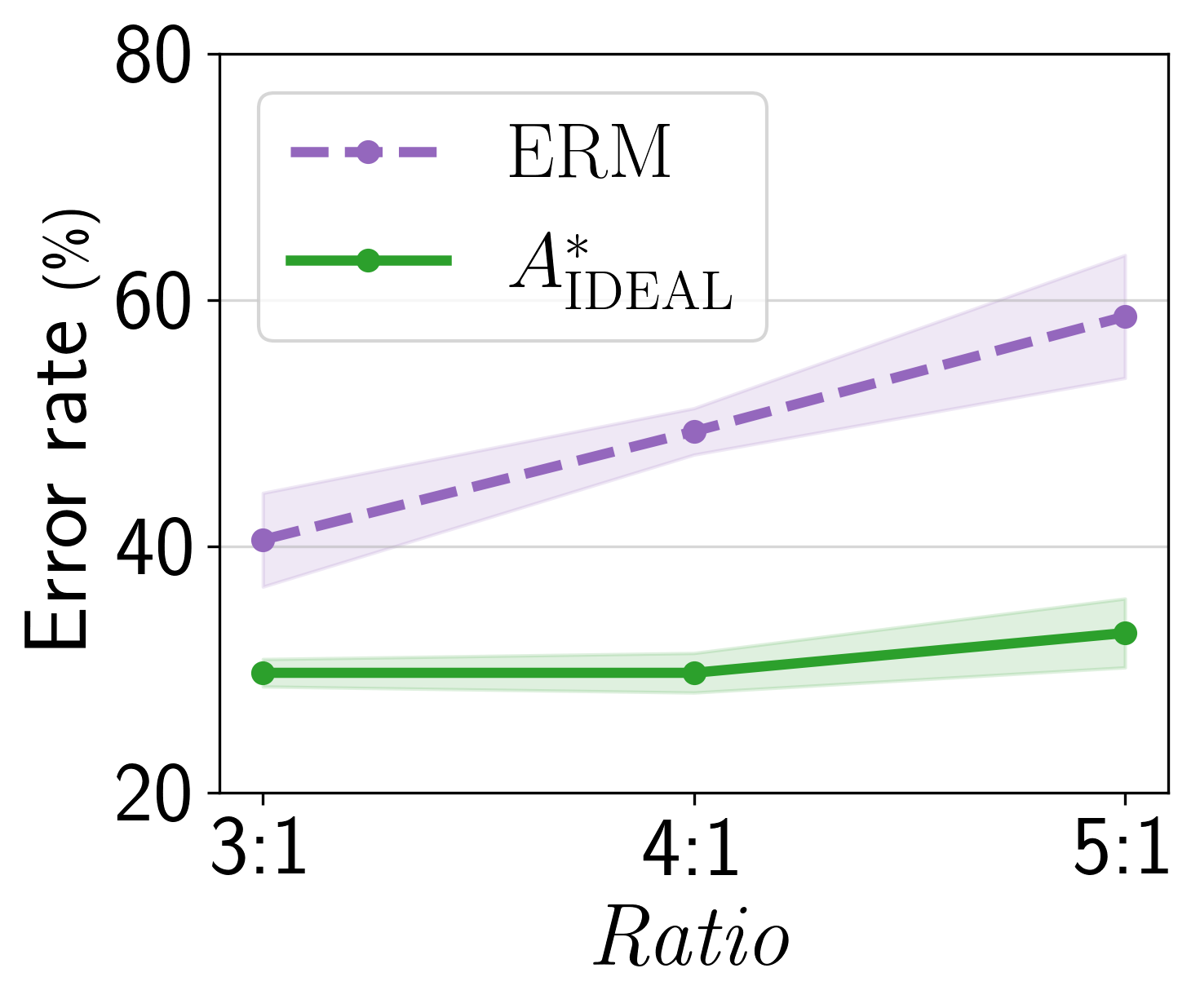}
  }
  \subfigure[]{
    \includegraphics[width=0.2\columnwidth]{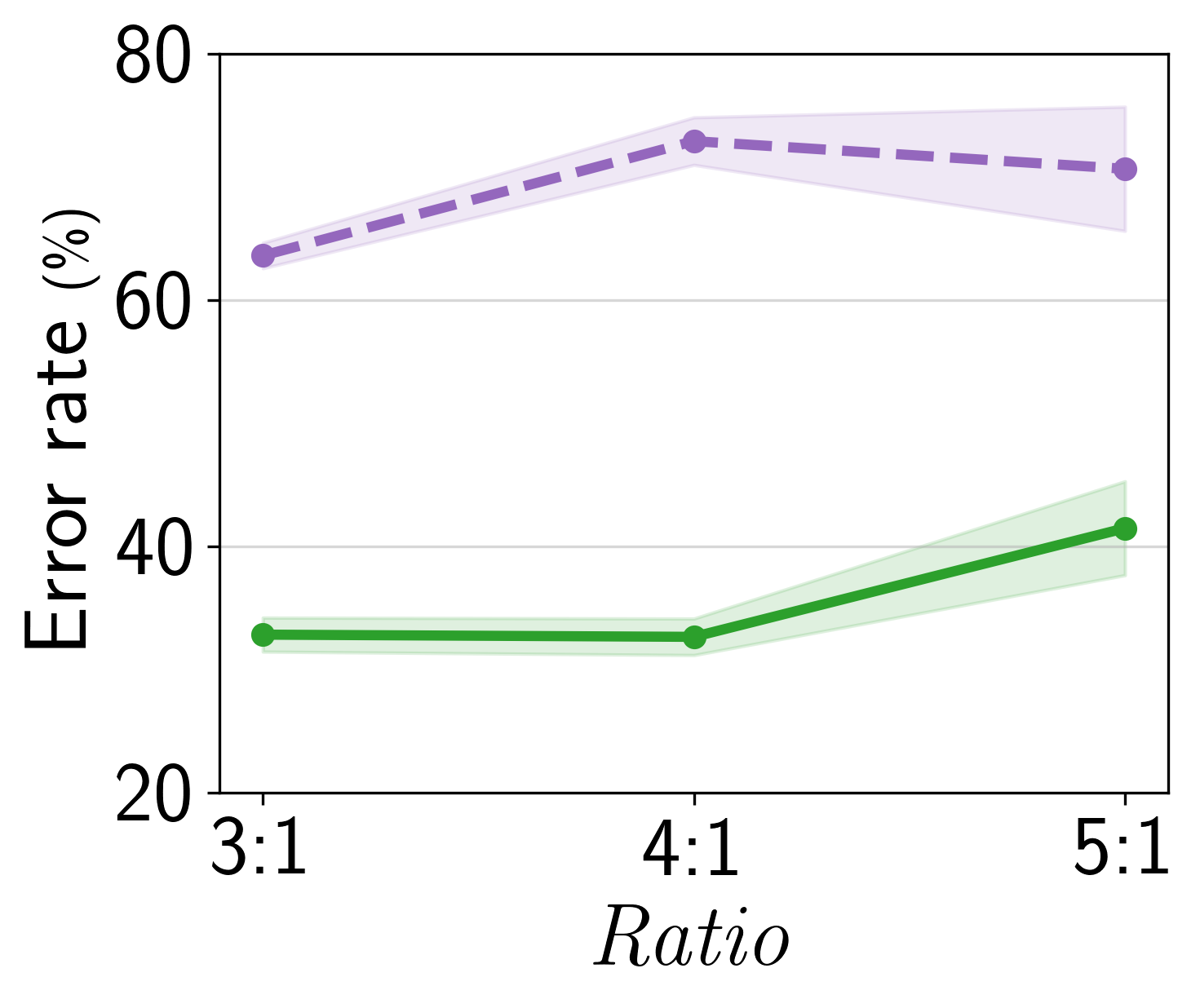}
  }
  \subfigure[]{
    \includegraphics[width=0.2\columnwidth]{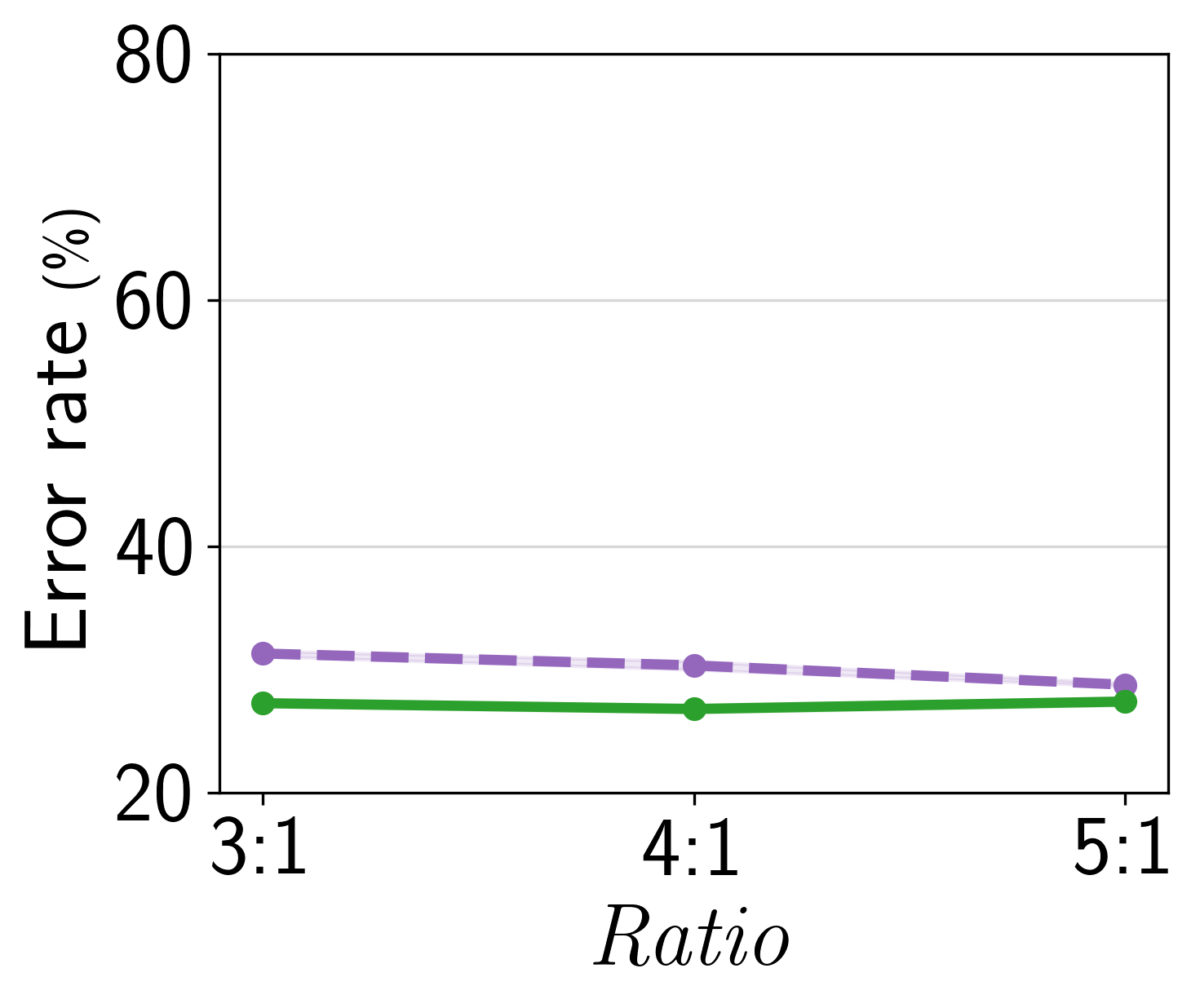}
  }
\caption{
(a) $\mathcal{R}^{\text{IDEAL}}_{A}$, (b) $\mathcal{R}^{\text{W+G}}_{A}$, and (c) $\mathcal{R}^{\text{AVG}}_{A}$ (\%) of ERM and best-performing algorithm $A^{*}_{\text{IDEAL}}$ when \textit{Scale} of SR-CMNIST is 2.
\label{c.2.2}
}
\end{figure}

\begin{figure}[th]
  \centering
  \subfigure[]{
    \includegraphics[width=0.2\columnwidth]{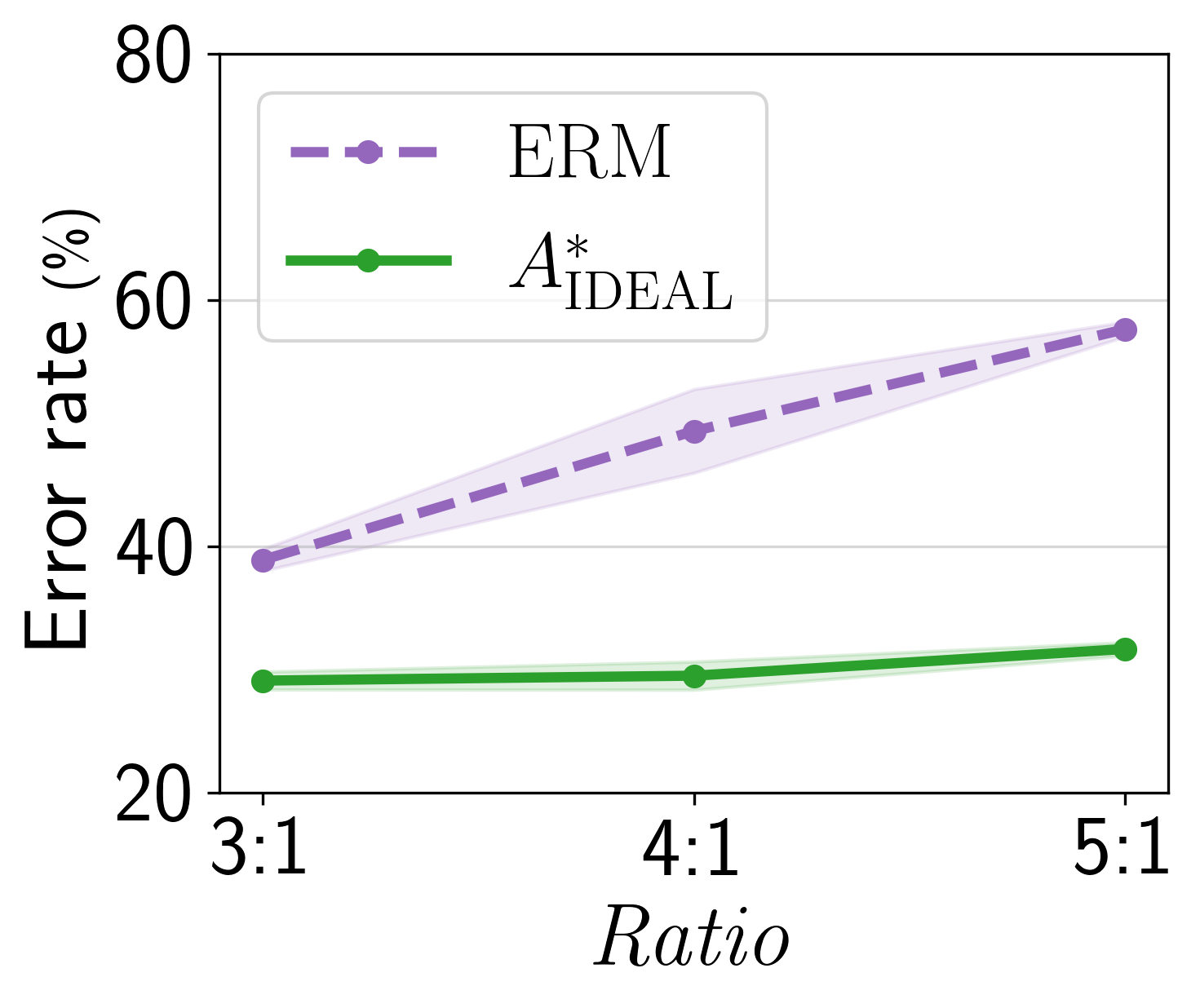}
  }
  \subfigure[]{
    \includegraphics[width=0.2\columnwidth]{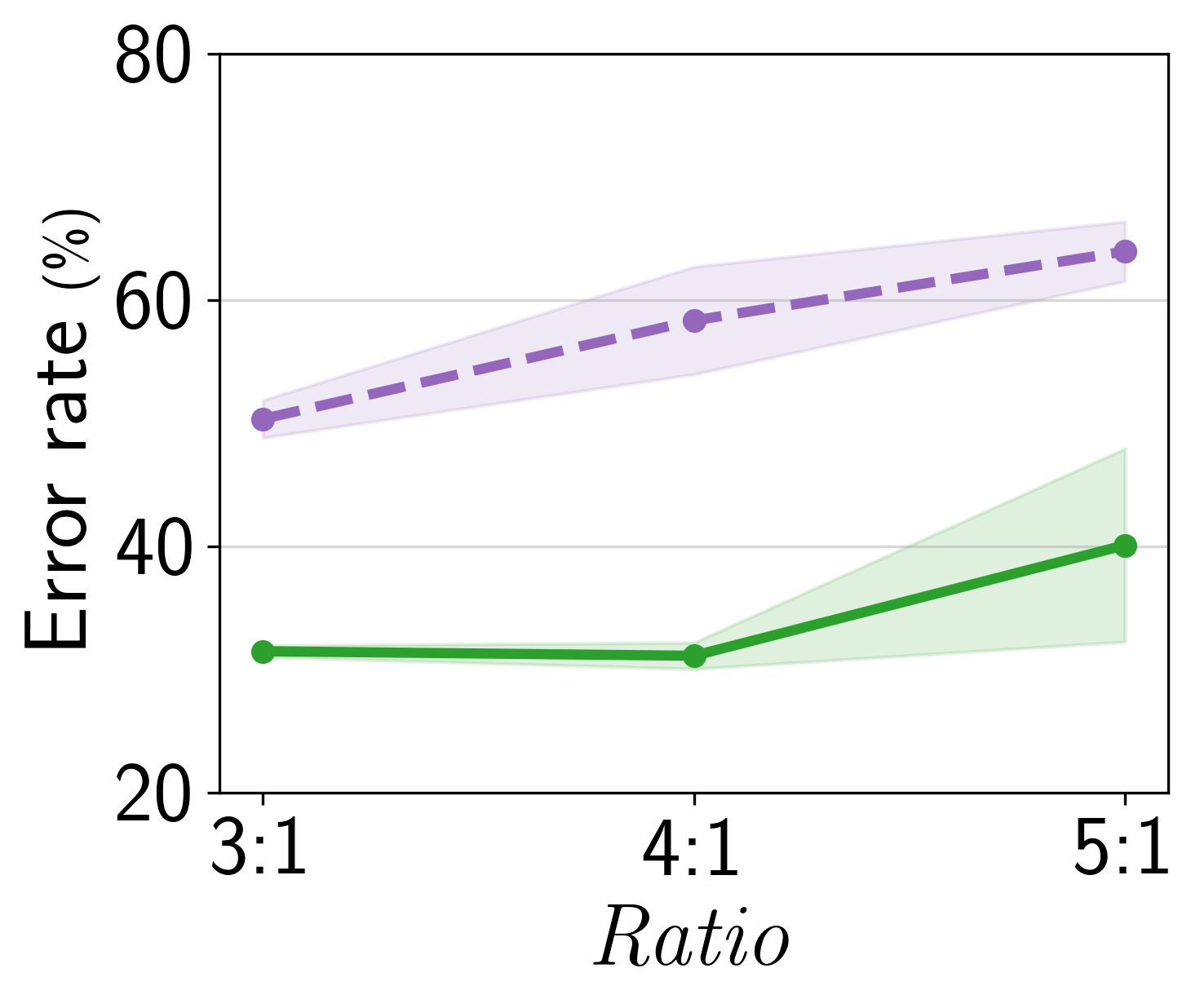}
  }
  \subfigure[]{
    \includegraphics[width=0.2\columnwidth]{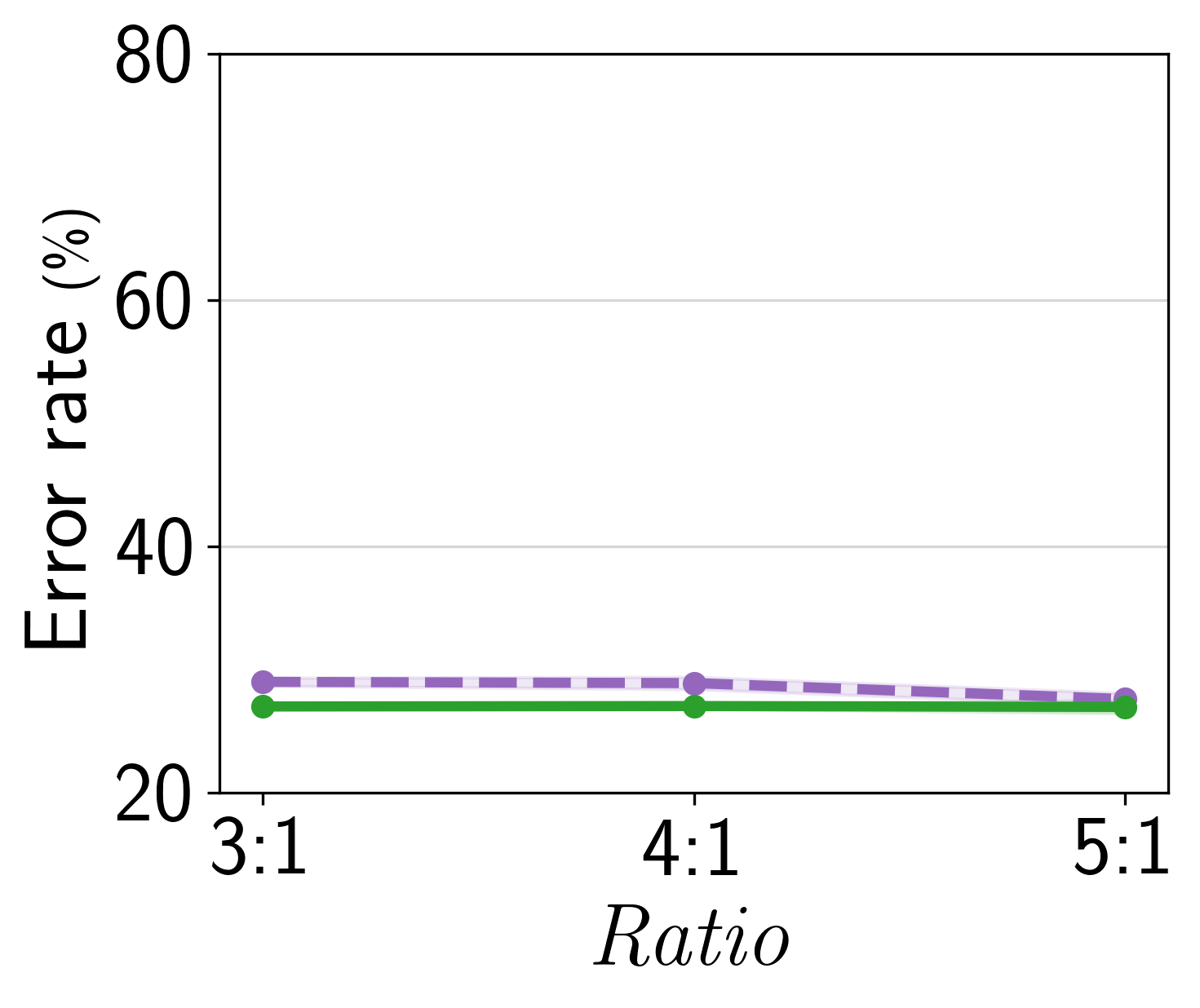}
  }
\caption{
(a) $\mathcal{R}^{\text{IDEAL}}_{A}$, (b) $\mathcal{R}^{\text{W+G}}_{A}$, and (c) $\mathcal{R}^{\text{AVG}}_{A}$ (\%) of ERM and best-performing algorithm $A^{*}_{\text{IDEAL}}$ when \textit{Scale} of SR-CMNIST is 3.
\label{c.2.3}
}
\end{figure}

\begin{figure}[th]
  \centering
  \subfigure[]{
    \includegraphics[width=0.2\columnwidth]{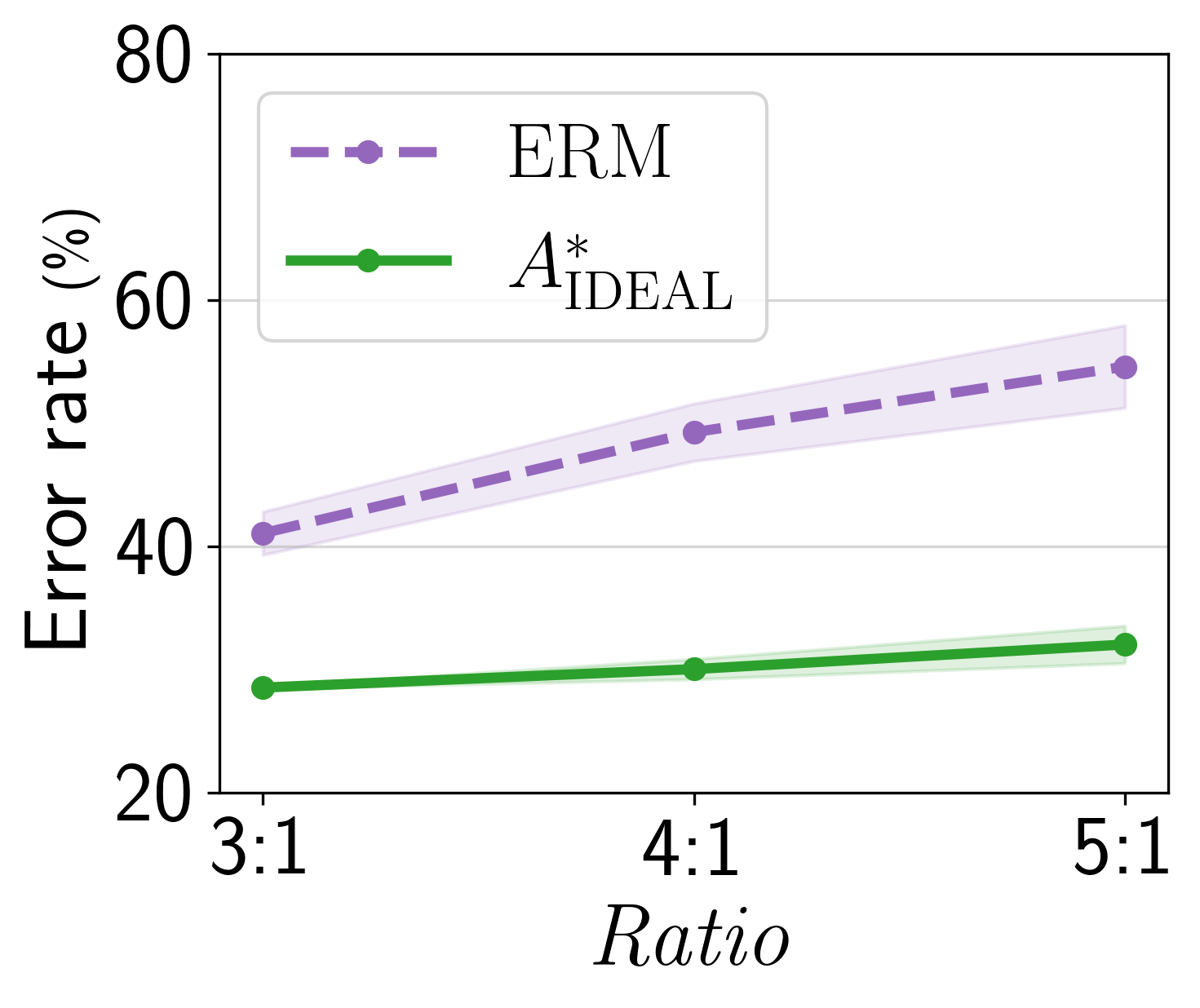}
  }
  \subfigure[]{
    \includegraphics[width=0.2\columnwidth]{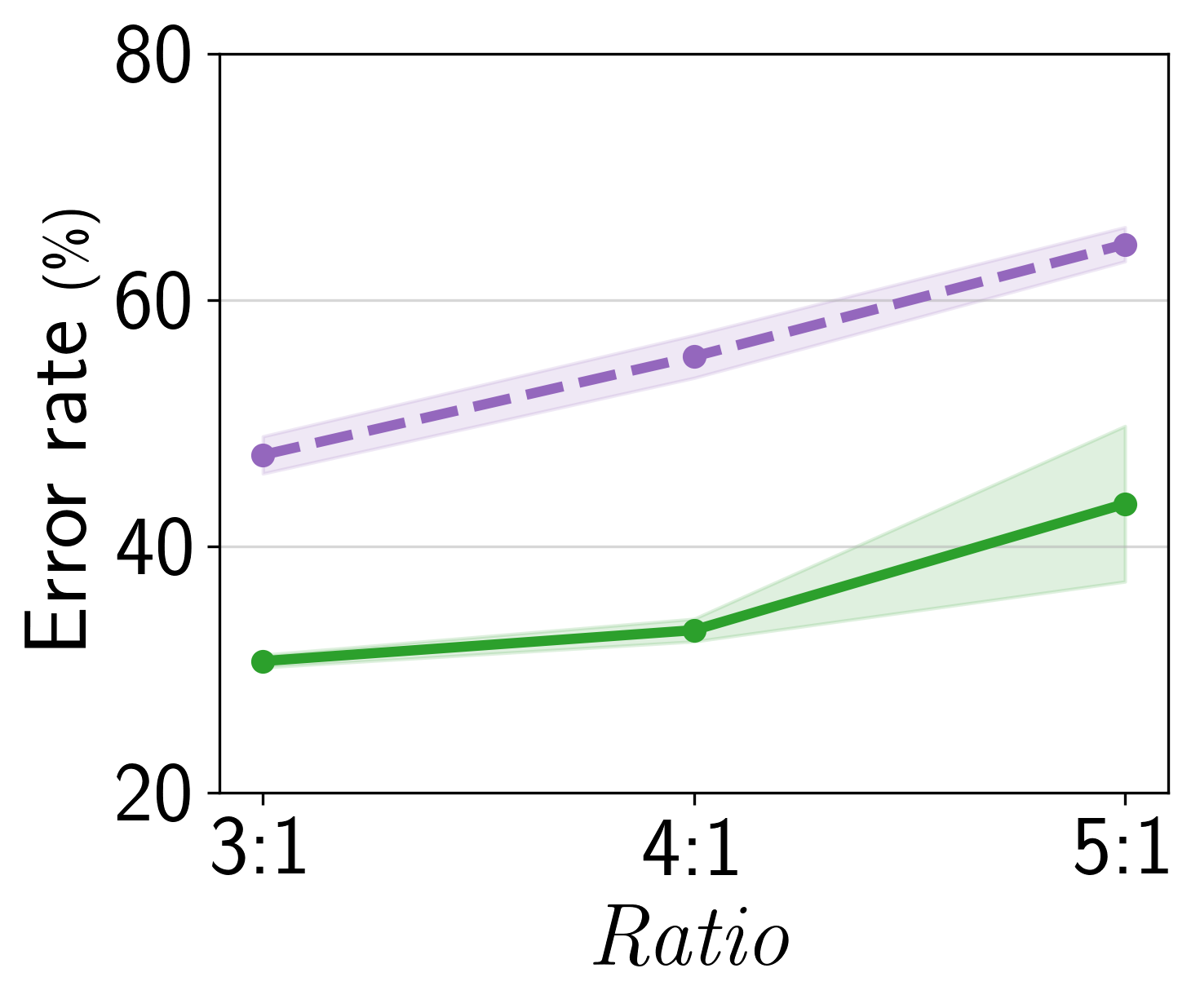}
  }
  \subfigure[]{
    \includegraphics[width=0.2\columnwidth]{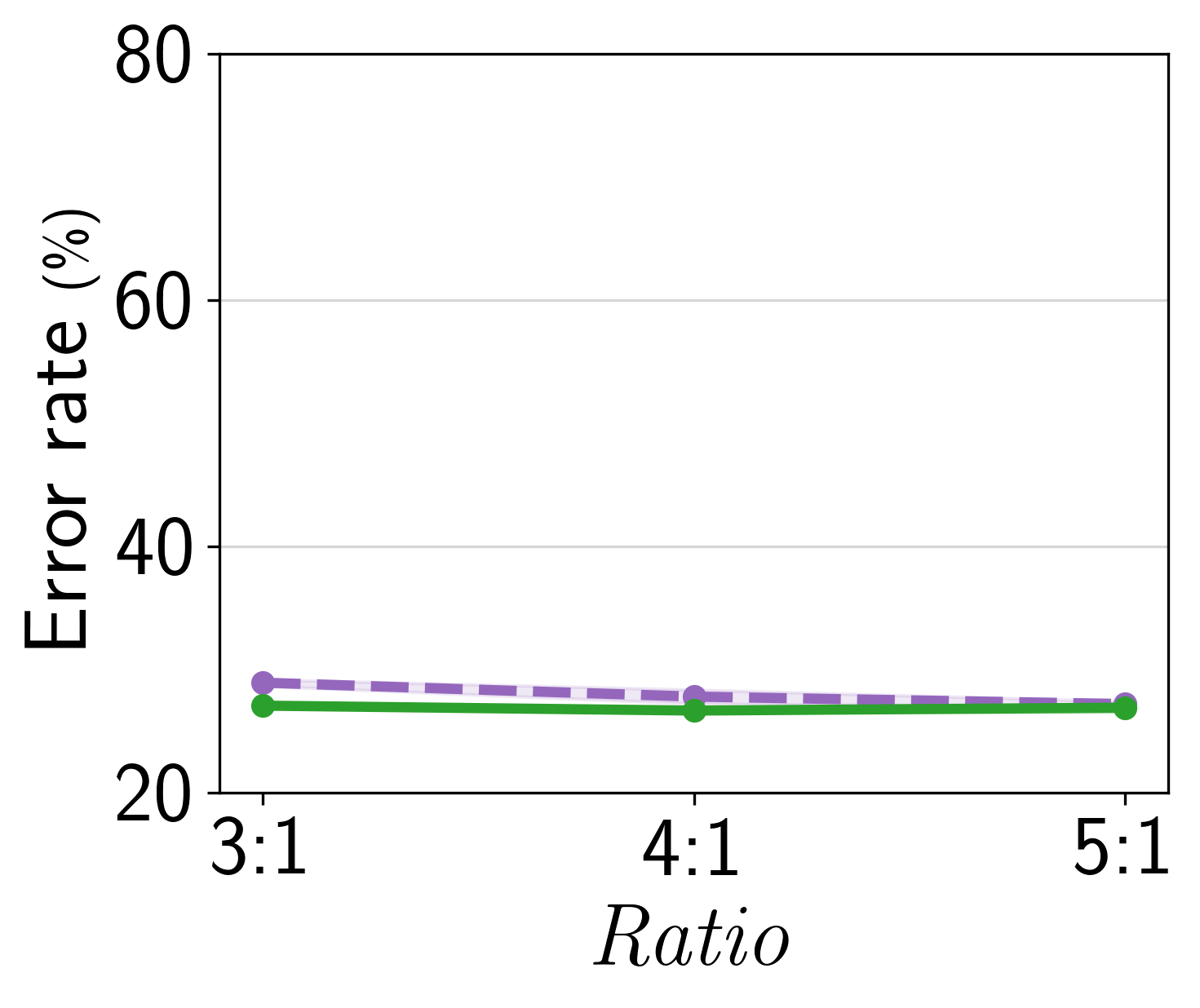}
  }
\caption{
(a) $\mathcal{R}^{\text{IDEAL}}_{A}$, (b) $\mathcal{R}^{\text{W+G}}_{A}$, and (c) $\mathcal{R}^{\text{AVG}}_{A}$ (\%) of ERM and best-performing algorithm $A^{*}_{\text{IDEAL}}$ when \textit{Scale} of SR-CMNIST is 4.
\label{c.2.4}
}
\end{figure}

\subsection{Selected algorithms from each measure}
\label{sec:alg_low}

\begin{table}[b!]
\begin{center}
\resizebox{\textwidth}{!}{
\begin{tabular}{cccccccccccccc}
\hline
\multirow{2}{*}{Scale} & \multirow{2}{*}{Ratio} & \multicolumn{3}{c}{Seed: 0}    & \multicolumn{3}{c}{Seed: 1}    & \multicolumn{3}{c}{Seed: 2}    & \multicolumn{3}{c}{Average}  \\ \cmidrule(l){3-5} \cmidrule(l){6-8} \cmidrule(l){9-11} \cmidrule(l){12-14}
                       &                        & $A^{*}_{\text{IDEAL}}$    & $A^{*}_{\text{W+G}}$       & $A^{*}_{\text{AVG}}$      & $A^{*}_{\text{IDEAL}}$    & $A^{*}_{\text{W+G}}$       & $A^{*}_{\text{AVG}}$      & $A^{*}_{\text{IDEAL}}$    & $A^{*}_{\text{W+G}}$       & $A^{*}_{\text{AVG}}$      & $A^{*}_{\text{IDEAL}}$    & $A^{*}_{\text{W+G}}$       & $A^{*}_{\text{AVG}}$      \\ \cmidrule(l){1-2} \cmidrule(l){3-5} \cmidrule(l){6-8} \cmidrule(l){9-11} \cmidrule(l){12-14}
\multirow{3}{*}{1}     & 3:1                    & GroupDRO & VREx     & VREx     & GroupDRO & GroupDRO & RSC      & GroupDRO & GroupDRO & RSC      & GroupDRO & GroupDRO & RSC      \\
                       & 4:1                    & GroupDRO & VREx     & VREx     & GroupDRO & GroupDRO & RSC      & GroupDRO & GroupDRO & GroupDRO & GroupDRO & GroupDRO & RSC      \\
                       & 5:1                    & VREx     & GroupDRO & MMD      & GroupDRO & GroupDRO & MMD      & GroupDRO & GroupDRO & MMD      & GroupDRO & GroupDRO & MMD      \\
\hline
\multirow{3}{*}{2}     & 3:1                    & VREx     & VREx     & VREx     & VREx     & GroupDRO & GroupDRO & GroupDRO & GroupDRO & GroupDRO & VREx     & GroupDRO & GroupDRO \\
                       & 4:1                    & GroupDRO & GroupDRO & GroupDRO & GroupDRO & GroupDRO & GroupDRO & GroupDRO & GroupDRO & GroupDRO & GroupDRO & GroupDRO & GroupDRO \\
                       & 5:1                    & GroupDRO & GroupDRO & MMD      & VREx     & GroupDRO & MMD      & VREx     & GroupDRO & MMD      & GroupDRO & GroupDRO & MMD      \\
\hline
\multirow{3}{*}{3}     & 3:1                    & GroupDRO & GroupDRO & ARM      & GroupDRO & GroupDRO & VREx     & GroupDRO & GroupDRO & GroupDRO & GroupDRO & GroupDRO & VREx     \\
                       & 4:1                    & GroupDRO & GroupDRO & VREx     & GroupDRO & GroupDRO & VREx     & GroupDRO & GroupDRO & VREx     & GroupDRO & GroupDRO & VREx     \\
                       & 5:1                    & VREx     & GroupDRO & VREx     & GroupDRO & GroupDRO & MLDG     & GroupDRO & VREx     & ARM      & GroupDRO & VREx     & VREx     \\
\hline
\multirow{3}{*}{4}     & 3:1                    & GroupDRO & GroupDRO & ARM      & GroupDRO & GroupDRO & VREx     & GroupDRO & GroupDRO & GroupDRO & GroupDRO & GroupDRO & VREx     \\
                       & 4:1                    & GroupDRO & GroupDRO & ARM      & GroupDRO & GroupDRO & ARM      & VREx     & VREx     & CDANN    & GroupDRO & GroupDRO & ARM      \\
                       & 5:1                    & GroupDRO & GroupDRO & VREx     & GroupDRO & VREx     & ARM      & VREx     & GroupDRO & ARM      & GroupDRO & VREx     & ARM      \\
\hline
\end{tabular}
}
\end{center}
\caption{
The selected algorithms, $A^{*}_{\text{IDEAL}}$, $A^{*}_{\text{W+G}}$, and $A^{*}_{\text{AVG}}$, that achieve the lowest value for each measure on various \textit{Scale} and \textit{Ratio} scenarios of SR-CMNIST.
\label{tab:best-selected-algorhtm}}
\end{table}

Table~\ref{tab:best-selected-algorhtm} shows the selected algorithms, $A^{*}_{\text{IDEAL}}$, $A^{*}_{\text{W+G}}$, and $A^{*}_{\text{AVG}}$, that achieve the lowest value for each measure on various \textit{Scale} and \textit{Ratio} scenarios of SR-CMNIST.
Out of total 36 combinations (12 \textit{Scale} and \textit{Ratio} scenarios; 3 different seeds for model training), $A^{*}_{\text{W+G}}$ is the same as $A^{*}_{\text{IDEAL}}$ in 25 combinations, whereas $A^{*}_{\text{AVG}}$ is the same as $A^{*}_{\text{IDEAL}}$ in only 9 combinations.  

When a single model is chosen with the average evaluation value of the three seeds, the resulting selected algorithms can be found to be as presented in the three rightmost columns. Out of 12 \textit{Scale} and \textit{Ratio} scenarios, $A^{*}_{\text{W+G}}$ is the same as $A^{*}_{\text{IDEAL}}$ for 9 scenarios, whereas $A^{*}_{\text{AVG}}$ is the same as $A^{*}_{\text{IDEAL}}$ for only 1 scenario.

\section{Conclusions}

Average measure has been the predominant choice for evaluating domain generalization algorithms. Our study challenges this conventional approach. We have identified the limitations of the average measure and proposed a more robust alternative – the worst+gap measure. Theoretical grounds were established through the derivation of two theorems, reinforcing the significance of our proposed measure. We devised five new domain generalization datasets for the purpose of DG measure study. Our experimental results substantiate the competitiveness of the worst+gap measure over the traditional average measure. This work not only questions the current convention of domain generalization research but also provides a promising avenue for enhancing domain generalization algorithms.

\section*{Acknowledgment}
\begin{minipage}[t]{\linewidth}
This work was supported by two National Research Foundation of Korea (NRF) grants funded by the Korea government (MSIT) (No.~NRF-2020R1A2C2007139, 2022R1A6A1A03063039) and in part by Institute of Information \& communications Technology Planning \& Evaluation (IITP) grant funded by the Korea government(MSIT) [NO.RS-2021-II211343, Artificial Intelligence Graduate School Program (Seoul National University)] and [No.RS-2023-00235293, Development of autonomous driving big data processing, management, search, and sharing interface technology to provide autonomous driving data according to the purpose of usage].
\end{minipage}

\bibliographystyle{unsrtnat} 
\bibliography{mybibfile}

\end{document}